\documentclass{article}

\usepackage{float}
\usepackage{fullpage, multicol}
\date{}
\usepackage{natbib}

\usepackage[utf8]{inputenc}
\usepackage[T1]{fontenc}

\usepackage{booktabs}
\usepackage{amsfonts}
\usepackage{nicefrac}
\usepackage{microtype}
\usepackage{xcolor}
\usepackage{enumitem}
\usepackage{subcaption}
\usepackage[many]{tcolorbox}
\usepackage{tabularx}
\usepackage{wrapfig}
\usepackage{algorithm}
\usepackage{algpseudocode}
\usepackage{amsmath,amsthm,amssymb,bbm}
\usepackage{mathtools}
\usepackage{thmtools, thm-restate}
\usepackage{float}

\usepackage{url}
\usepackage{hyperref}

\usepackage[dvipsnames]{xcolor}

\usepackage{algorithm}
\usepackage{algpseudocode}
\usepackage{amsmath,amsthm,amssymb,bbm}
\usepackage{mathtools}
\usepackage{thmtools, thm-restate}
\allowdisplaybreaks

\DeclareDocumentCommand \norm { o m }{{\lVert #2 \rVert_#1}}

\newcommand{\mdp}{\mathcal{M}}
\newcommand{\states}{\mathcal{S}}
\newcommand{\actions}{\mathcal{A}}
\newcommand{\rew}{R}
\newcommand{\transitions}{P}

\DeclareFontFamily{OT1}{pzc}{}
\DeclareFontShape{OT1}{pzc}{m}{it}{<-> s * [1.10] pzcmi7t}{}
\DeclareMathAlphabet{\pzccal}{OT1}{pzc}{m}{it}
\newcommand{\algo}{\pzccal{A}}

\theoremstyle{plain}
\newtheorem{theorem}{Theorem}[section]
\newtheorem*{theorem*}{Theorem}

\theoremstyle{definition}
\newtheorem{definition}[theorem]{Definition}

\theoremstyle{remark}

\title{Behavior-Consistent Deep Reinforcement Learning}

\author{
  Marcel Hussing\thanks{Equal contribution} \\
  University of Pennsylvania \\
  \texttt{mhussing@seas.upenn.edu} \\
  \and
  Liv d'Aliberti\footnotemark[1] \\
  Princeton University \\
  \texttt{od2961@princeton.edu} \\
  \and
  Claas Voelcker \\
  University of Texas at Austin \\
  \and
  Ben Eysenbach \\
  Princeton University \\
  \and
  Eric Eaton \\
  University of Pennsylvania \\
}

\begin{document}

\maketitle

\begin{abstract}
    Reinforcement learning (RL) often exhibits high variance across training runs, leading to unreliable performance and posing a major challenge to deployment in real-world domains. 
    In this work, we address the challenge of cross-run policy divergence by formalizing the problem of {\em behavior-consistent RL}, where the objective is to obtain policies that are both high-performing and distributionally similar across training runs. 
    Our key observation is that maximum-entropy RL provides a direct mechanism for controlling behavioral divergence by anchoring runs to a common (uniform) prior. 
    We prove that, for Boltzmann policies, choosing the temperature proportional to $Q$-function disagreement bounds the pairwise KL divergence between the induced policies. 
    However, we also show that na\"{\i}vely increasing entropy might impair policy optimization while amplifying off-policy error.
    Building upon these observations, we propose $Q$-value Expectile Disagreement (QED), a state-dependent temperature schedule that uses double-critic disagreement as a single-run proxy for cross-run disagreement. 
    Empirically, we demonstrate that across 18 continuous-control tasks, QED reduces across-run divergence by two orders of magnitude without sacrificing performance, resulting in a considerable reduction in return variance at modest sample-efficiency costs. 
\end{abstract}

\section{Introduction}
\label{sec:introduction}

Reinforcement learning (RL) is a powerful tool for solving control problems. Yet, the lack of consistency across independent training runs remains a persistent and under-addressed challenge~\citep{colas2018how, henderson2018matters, paleu2023reproducibility}. 
Small changes in random initialization, data ordering, or environment stochasticity can yield policies that differ dramatically in performance~\citep{islam2017reproducibility, henderson2018matters, bjorck2022is}. 
Even when final performance returns appear similar, learned policies can exhibit different behaviors~\citep{clary2018variability, flageat2024beyond, eaton2023replicable, eaton2026replicable}, leading to variations in robustness and failure modes. 

This variability poses a fundamental obstacle to both scientific progress~\citep{henderson2018matters} and practical deployment~\citep{lynnerup2020a, cavenaghi2023a}. 
From a scientific perspective, high variance across runs complicates evaluation and often requires averaging over numerous random seeds to obtain statistically significant comparisons~\citep{colas2018how, agarwal2021deep}. 
Stochasticity can even cause performance to vary for a fixed policy~\citep{flageat2024beyond}. 
In practice, the consequences are more severe. 
When policies trained under identical conditions yield qualitatively different behavior, iterative processes such as reward design~\citep{booth2023theperils}, debugging, and safety validation become unreliable, as apparent improvements may reflect stochastic idiosyncrasies rather than the modification.  
For example, in deployment settings such as large-language-model fine-tuning, policies trained on such rewards might cause regular updates to feel qualitatively different to users in terms of helpfulness, tone, or safety. 

While prior work has studied RL instability through statistical metrics such as return variance~\citep{colas2018how, agarwal2021deep}, behavioral divergence across runs has received comparatively little algorithmic attention (see
Section~\ref{sec:related}
for related work).
Existing algorithmic interventions that reduce variability across RL training runs or during learning remain sparse~\citep{bjorck2022is, tang2024improving}.
Related theoretical work formalizes replicability as requiring identical outputs across executions~\citep{eaton2023replicable, eaton2026replicable}.
\citet{eaton2026replicable} also provide an empirical method inspired by their theory that obtains increased cross-execution policy agreement in a small set of Atari experiments. We go beyond this by studying how behavioral consistency can be achieved under the practical requirements of any deep RL setting, including those of continuous control.

In many applications, from robotics to conversational AI systems, the consistency of behavior is often as critical as its performance.  
In this work, we argue that {\bf behavioral consistency should be treated as a core objective in RL}. We formalize this perspective through the new framework of {\em behavior-consistent RL}, in which the goal is to learn policies that are both high-performing and distributionally similar across independent runs. This framing shifts the focus of RL from solely achieving high expected return to ensuring that learning outcomes are stable and replicable. 

\begin{figure}[t!]
  \centering

  \begin{minipage}[t]{0.49\textwidth}
    \centering
    \begin{tcolorbox}[
      colframe=BrickRed,
      colback=BrickRed!3,
      boxrule=0.8pt,
      sharp corners,
      width=\textwidth,
      boxsep=0pt,
      enhanced,
      left=1pt,
      right=1pt,
      top=4pt,
      bottom=0pt,
      equal height group=policyboxes,
      valign=top
    ]
      \centering
      \textbf{Standard Entropy Tuning}\\[-2pt]

      \includegraphics[width=\linewidth]{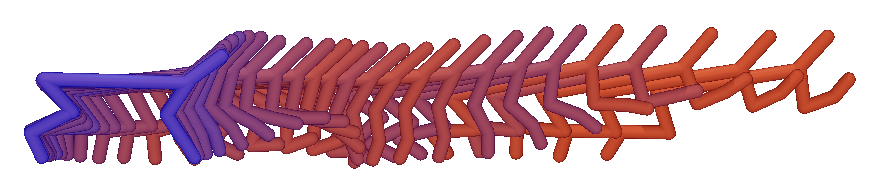}
      \includegraphics[width=\linewidth]{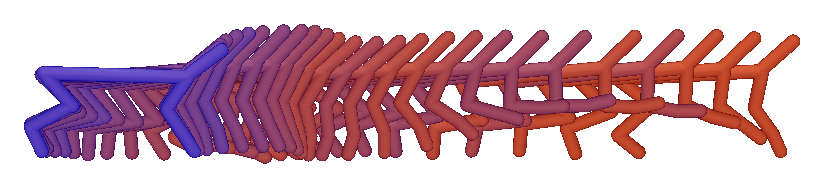}
      \includegraphics[width=\linewidth]{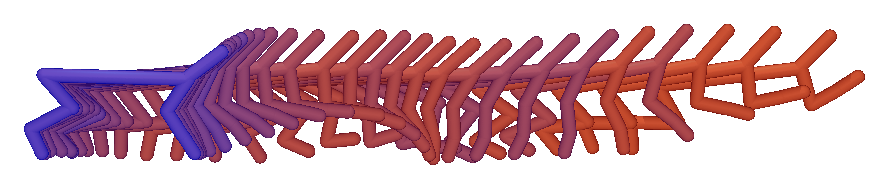}
    \end{tcolorbox}
  \end{minipage}
  \hfill
  \begin{minipage}[t]{0.49\textwidth}
    \centering
    \begin{tcolorbox}[
      colframe=Blue,
      colback=Blue!3,
      boxrule=0.8pt,
      sharp corners,
      width=\textwidth,
      boxsep=0pt,
      enhanced,
      left=1pt,
      right=1pt,
      top=4pt,
      bottom=0pt,
      equal height group=policyboxes,
      valign=top
    ]
      \centering
      \textbf{$Q$-value Expectile Disagreement (Ours)}\\[-2pt]

      \includegraphics[width=\linewidth]{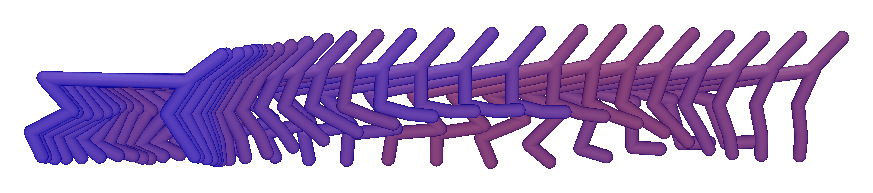}
      \includegraphics[width=\linewidth]{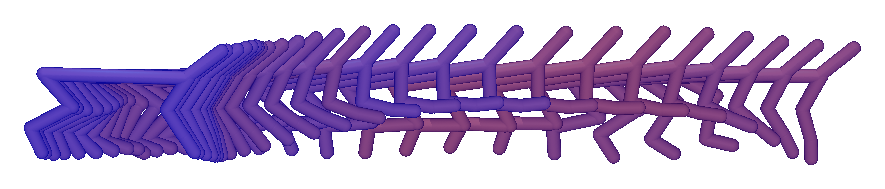}
      \includegraphics[width=\linewidth]{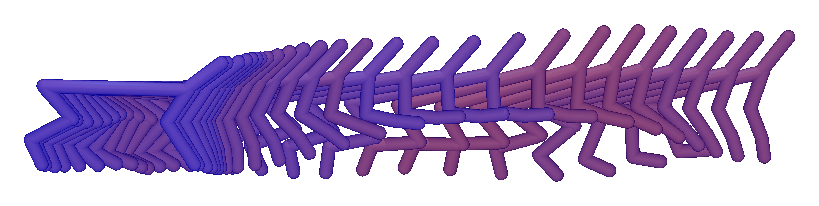}
    \end{tcolorbox}
  \end{minipage}

  \caption{\textbf{QED makes independently trained policies visibly behavior-consistent.} Visualization of policies from three different training runs on the $\texttt{cheetah\_run}$ task, comparing (left) traditional entropy autotuning~\citep{haarnoja2018applications} against (right) our approach (QED). Color shade denotes the mean pairwise $L_2$ distance between state vectors at each timestep: blue is low, red is high. 
  }
  \label{fig:visualization}
\end{figure}

Toward this goal, our key insight is that maximum entropy (MaxEnt) RL provides a natural mechanism for inducing behavioral consistency. 
By regularizing policies toward a common prior~\citep{ziebart2008maximum,levine2018reinforcement}, entropy shrinks the policy space and restricts the range of behaviors independent runs can realize. 
Yet, na\"{i}vely increasing entropy can amplify off-policy  error. These competing effects reveal a tradeoff: maximizing entropy promotes consistency but can degrade performance. To resolve this tradeoff, we establish a theoretical connection between policy divergence and $Q$-disagreement. We prove that, for Boltzmann policies, scaling the temperature in proportion to $Q$-disagreement bounds the divergence between the induced policies.
This suggests the following principled approach: entropy should be high when $Q$-estimates are uncertain and decrease as they converge. 

Based on this insight, we introduce $Q$-value Expectile Disagreement (QED), a practical algorithm for adaptive entropy tuning in RL. 
QED increases temperature early to achieve consistency and cools down over time for convergence.
We evaluate QED on top of two strong base algorithms and show that it reduces pairwise policy divergence across independent runs by up to two orders of magnitude while maintaining competitive performance, producing quantitatively and qualitatively more similar behavior, and reducing return variance by around $50$\%. 
These results show that controlling behavioral consistency in RL is both feasible and beneficial, offering a new perspective on stability and replicability in RL. \textcolor{Blue}{\textbf{Our key contributions include: }}
\vspace{-5pt}
\begin{enumerate}[
    leftmargin=14pt,
    noitemsep,
    nolistsep,
    topsep=0pt,
    parsep=0pt,
    label=\textcolor{Blue}{\textbf{\arabic*.}}
]
    \item We introduce the setting of \emph{behavior-consistent reinforcement learning} (BRL).
    \item We prove a relationship between policy divergence and critic disagreement in MaxEnt RL.
    \item We illustrate the off-policy challenges in high-entropy MaxEnt RL.
    \item We introduce \textsc{QED}, a method that increases behavioral similarity and reduces return variance.
\end{enumerate}

\section{Preliminaries}
\label{sec:background}

We consider the RL problem~\citep{sutton2018introduction} of finding an optimal policy in a Markov decision process (MDP)~\citep{puterman1994markov} \( \mdp = (\states,\actions, \transitions, \rew,\gamma)\), with states $\states$, actions $\actions$, transition function \(\transitions:  \states \times \actions \to \Delta(\states\)),  reward function \(\rew: \states\times\actions\to\mathbb R\), and discount factor \(\gamma\in(0,1)\). In general, this paper focuses on continuous state-action spaces but for ease of exposition we will state all definitions and results with discrete action spaces. 

A policy \(\pi:\states\to\Delta(\actions)\) attains \( J(\pi)
=\mathbb E_{\tau\sim\pi}[\sum_{t=0}^\infty\gamma^t\,r(s_t,a_t)],
\) and the optimal policy is \(\pi^\star=\arg\max_\pi J(\pi)\).  
In practice, one learns the action-value function  
\[
Q^\pi(s,a)
=\mathbb E_{\tau\sim\pi}\Bigl[\sum_{t=0}^\infty\gamma^t\,r(s_t,a_t)\,\bigm|\,s_0=s,a_0=a\Bigr]
\]
and acquires $\pi$ via approximate policy improvement with respect to $Q^\pi$, noting that the $Q^\pi$-greedy policy $\pi_{\mathrm{greedy}}(s)\in\arg\max_{a\in\actions} Q^\pi(s,a)$ is optimal when $Q^\pi=Q^\star$. 

\paragraph{KL-Constrained RL.}
KL-regularized RL~\citep{todorov2009efficient, peters2010relative} provides a principled surrogate for this idea by constraining each policy improvement step relative to a known \emph{reference} (or \emph{prior}) policy density $\pi_0(\cdot\mid s)$.
A standard formulation is the KL-constrained objective
\begin{equation*}
    \max_{\pi}\;\;
    \mathbb E_{\tau\sim\pi}\Bigl[\sum_{t=0}^\infty \gamma^t\Bigl(
    r(s_t,a_t)\;-\;\alpha\,D_{\mathrm{KL}}\!\bigl(\pi(\cdot\mid s_t)\,\|\,\pi_0(\cdot\mid s_t)\bigr)\Bigr)\Bigr],
\end{equation*}
with $\alpha>0$ controlling the strength of the constraint.
This induces a prior-weighted Boltzmann policy: for each state $s$, the optimal action probabilities are proportional to the reference density
\begin{equation*}
    \pi^\star(a\mid s)
    =
    \frac{\pi_0(a\mid s)\exp(Q^\star(s,a)/\alpha)}
    {\sum_{a' \in \actions}\pi_0(a'\mid s)\exp(Q^\star(s,a')/\alpha)\ } \enspace.
\end{equation*}

\paragraph{Maximum Entropy RL.}
In this paper, we work with the maximum entropy (MaxEnt)~\citep{ziebart2008maximum, haarnoja2018sac} framework as an instantiation of KL-regularized RL~\citep{levine2018reinforcement}, obtained by choosing a state-independent base density $\pi_0$ on $\actions$, so that the KL penalty differs from the entropy penalty $-\!H(\pi(\cdot\mid s))$ only by an additive constant.
The resulting objective is
\begin{equation*}
    \pi^\star_{\rm ME}
    =\arg\max_\pi
    \;\mathbb E_{\tau\sim\pi}
    \Bigl[\sum_{t=0}^\infty\gamma^t\bigl(r(s_t,a_t)+\alpha\,H(\pi(\cdot\!\mid\!s_t))\bigr)\Bigr],
\end{equation*}
where $H(\pi(\cdot \mid s)) = -\mathbb E_{a\sim\pi(\cdot\mid s)} \bigl[\log\pi(a \mid s)\bigr]$ is the differential entropy, and the optimal policy has the Boltzmann form $\pi^\star(a\mid s)\propto \exp(Q^\star(s,a)/\alpha)$. Informally, larger \(\alpha\) yields a more uniform action distribution, while smaller \(\alpha\) concentrates probability on high-value actions.
When $\alpha$ is fixed to a specific value $C$, we will simply denote it $\alpha_{C}$.
More traditionally, $\alpha$ is tuned such that the action distribution approaches a fixed target entropy $H_{\text{target}}$~\citep{haarnoja2018sac} (often set to the negative of the dimension of the action space) by continually fitting $\alpha$ to find
\begin{equation} \label{eq:sac_alpha_loss}
    \alpha_{\rm TE}
    \in
    \arg\min_{\alpha\in\mathbb R_{>0}}
    \mathbb E_{s\sim\mathcal D,\;a\sim\pi(\cdot\mid s)}
    \Bigl[
    \alpha\bigl(-\log \pi(a\mid s)- H_{\text{target}} \bigr)
    \Bigr].
\end{equation}

\section{Behavior-consistent RL}

\label{sec:behavior_consistent_rl}

Moving beyond the standard RL problem---and towards empirical algorithms that make reward tuning, experimental replication, and debugging more reliable---we study differences between independent executions of the same algorithm in a fixed MDP, which we refer to as \emph{runs}. To capture behavioral similarity between the policies produced by these runs, we introduce the setting of \emph{behavior-consistent RL} (BRL). In BRL, the objective is not only to maximize return but also to obtain policies that are distributionally similar across runs, so that an algorithm is both performant and produces policies that exhibit the same behavior. 

A common way to measure the similarity between the two probability distributions---or, in our case, policies---is via a divergence $D$. Consider the disagreement between any two policies 
$\mathbb{E}_{s\sim \mu_{\mathrm{ref}}}
\left[
    D
    \left(
        \pi^{(i)}(\cdot\mid s),
        \pi^{(j)}(\cdot\mid s)
    \right)
\right]$,
where $\mu_{\mathrm{ref}}$ is a reference distribution. 
To turn agreement into a measure of behavioral stability, one can tie the reference distribution $\mu_{\mathrm{ref}}$ directly to the output of the algorithm.
A sensible choice for it is the population state distribution induced by the learned policies themselves,
    $\mu_{\mathrm{ref}}
    =
    \mathbb{E}_{\pi \sim \algo}
    \left[
        d^{\pi}
    \right],$
where $\pi \sim \algo$ denotes the policy produced by a random execution of algorithm $\algo$, and $d^{\pi}$ denotes the state occupancy distribution induced by executing $\pi$ in the MDP.
Then, to measure behavioral similarity, the following criterion can be defined. 

\begin{definition}[Behavior-consistency in RL]
\label{def:behavior-consistent}
    Let $\algo$ be an RL algorithm that outputs a policy $\pi \in \Pi$ when trained on a fixed MDP.
    Let $\pi \sim \algo$ denote the random policy obtained from one execution of $\algo$, and let $d^{\pi}$ denote the state occupancy distribution induced by executing $\pi$ in the MDP.
    Define the behavioral reference distribution as
        $\mu_{\pi}
        \doteq
        \mathbb{E}_{\pi \sim \algo}
        \left[
            d^{\pi}
        \right].$
    Let $D:\Delta(\actions)\times\Delta(\actions)\to\mathbb{R}_{\geq 0}$ be a non-negative divergence between action distributions.
    Let $\pi^{(i)},\pi^{(j)}\sim\algo$ denote two policies produced by independent executions of $\algo$.
    The behavioral consistency of $\algo$ is defined as
    \begin{equation}
        \mathcal{V}(\algo)
        \doteq
        \mathbb{E}_{\pi^{(i)},\pi^{(j)}\sim\algo}
        \left[
        \mathbb{E}_{s\sim\mu_{\pi}}
        \left[
            D
            \left(
                \pi^{(i)}(\cdot\mid s),
                \pi^{(j)}(\cdot\mid s)
            \right)
        \right]
        \right].
    \end{equation}
\end{definition}
This notion captures the requirement that all policies produced by the same algorithm be close only on the space said policies will visit. This is sufficient to attain that lower inter-run variability means that different runs produce more behaviorally consistent policies. 
We note two things. 
First, the measure is directly dependent on the policies produced by the algorithm and thus two algorithms might measure similarity on different distributions. 
Second, behavior-consistency should always be tied to a secondary criterion such as return maximization as otherwise the trivial solution of outputting a constant policy will minimize the requirement.

\section{High-temperature maximum entropy RL} \label{sec:high_entropy}

If the policy of another run were available, we could directly constrain the updates in our current run toward it using KL-constrained RL and achieve behavioral similarity.
Yet, in practice, the policies obtained from other runs may not be available, either because obtaining them might require additional expensive training or because they are produced independently by other users.
Instead, we use a fixed, common reference policy \(\pi_0\) as an independent anchor toward which all runs can be regularized.
If each policy remains close to this shared reference, then behavior consistency can be controlled without access to other runs' policies; all policies remain close to the reference and therefore to each other. 
The MaxEnt RL framework offers a natural common prior: the uniform policy. 

We now prove that the Boltzmann temperature, and therefore the strength of the regularization, directly controls the difference between policies.
With an appropriate temperature, the soft Bellman operator produces policies across runs that never differ by more than some constant $\kappa$ even if their $Q$-values are not initialized identically. 
Intuitively, the only way two Boltzmann policies can substantially differ at a fixed state is if the corresponding $Q$-vectors differ relative to the temperature. 
When this happens, increasing the entropy regularization reduces the softmax sensitivity to pointwise $Q$-value discrepancies, preventing the resulting action probabilities from separating too sharply.
Setting $\alpha(s)$ proportional to the maximum pointwise disagreement makes these discrepancies small in temperature units, which keeps the two policies close in KL.
The following theorem formalizes this.

\begin{theorem}[Pairwise KL control via disagreement-scaled temperature] \label{thm:pairwise-kl}
    Assume $\actions$ is finite and fix a state $s\in\states$.
    Let $Q^{(1)}(s,\cdot),Q^{(2)}(s,\cdot)\in\mathbb R^{|\actions|}$ be two action-value vectors, and fix $\kappa>0$ and $\alpha_{\min} > 0$.
    Define the shared temperature
    \begin{equation} \label{eq:disagreement_alpha}
    \alpha(s) \doteq \max\left\{\alpha_{\min}, \frac{\|Q^{(1)}(s,\cdot)-Q^{(2)}(s,\cdot)\|_\infty}{\kappa}\right\}.
    \end{equation}
    Let $\pi^{(1)}(\cdot\mid s)$ and $\pi^{(2)}(\cdot\mid s)$ be the corresponding Boltzmann policies at temperature $\alpha(s)$. Then
    \begin{equation*}
    D_{\mathrm{KL}}(\pi^{(1)}(\cdot\mid s)\|\pi^{(2)}(\cdot\mid s))\le 2\kappa \enspace.
    \end{equation*}
\end{theorem}

{\em Proof Sketch}: (Full proof in Appendix~\ref{app:kl_proof}) The high-level idea is to write the KL of two Boltzmann policies as an average of the log-ratio between them. This log-ratio has two pieces: how much the two $Q$-values differ on the chosen action, and how much their normalizing constants differ. The first piece is immediately bounded by the maximum $Q$-disagreement. The second piece is bounded by observing that every term in one normalizer is within a fixed factor of the corresponding term in the other normalizer. So both pieces are controlled by the maximum disagreement divided by the temperature. Since the temperature was chosen appropriately, the total KL is bounded. $\qed$

Next, we show that using the disagreement-scaled temperature can lead to a behaviorally-consistent optimal entropy-regularized policy. In fact, we recover the optimal soft-value iteration result as the disagreement-dependent error vanishes with convergence.

\begin{theorem}[Convergence under disagreement-scaled temperature]  \label{thm:disagreement_temperature_convergence}
Assume $\alpha_{\min}>0$. Let $Q^\star$ denote the optimal Q-function of the unregularized MDP. For a temperature $\alpha:\states\to\mathbb R_{>0}$, define
\begin{equation*}
(\mathcal T_\alpha Q)(s,a)
\doteq
r(s,a)
+
\gamma\mathbb E_{s'\sim P(\cdot\mid s,a)}
\left[
\alpha(s')\log\sum_{a'\in\actions}
\exp\left(\frac{Q(s',a')}{\alpha(s')}\right)
\right].
\end{equation*}
Consider the coupled iterates $Q_{t+1}^{(i)}=\mathcal T_{\alpha_t}Q_t^{(i)}$, $i\in\{1,2\}$, where $\alpha_t(s)$ is the coupling temperature from \eqref{eq:disagreement_alpha}. Let $\Delta_0\doteq \|Q_0^{(1)}-Q_0^{(2)}\|_\infty$ be the initial disagreement. Then, for $i\in\{1,2\}$ and $t\ge 0$,
\begin{equation*}
\|Q_t^{(i)}-Q^\star\|_\infty
\le
\gamma^t\|Q_0^{(i)}-Q^\star\|_\infty
+
\textcolor{BrickRed}{\frac{\gamma\alpha_{\min}\log|\actions|}{1-\gamma}(1-\gamma^t)}
+
\textcolor{Blue}{\frac{\Delta_0\log|\actions|}{\kappa}\,t\gamma^t}\enspace.
\end{equation*}
\end{theorem}

{\em Proof Sketch}: The proof (Appendix~\ref{app:convergence_proof}) follows the standard approximate value-iteration template for the soft Bellman operator~\citep{ziebart2008maximum, haarnoja2018sac, levine2018reinforcement}. The key step is to relate the approximation error to our adaptive temperature. At iteration $t$, the soft backup differs from the unregularized Bellman backup by at most a term proportional to $\alpha_{\min}+\|Q_t^{(1)}-Q_t^{(2)}\|_\infty/\kappa$. As the two coupled runs use the same temperature, their disagreement contracts as $\|Q_t^{(1)}-Q_t^{(2)}\|_\infty\le \gamma^t\Delta_0$. Thus, the \textcolor{Blue}{disagreement-dependent error} vanishes after unrolling the recursion, leaving only the persistent \textcolor{BrickRed}{entropy-regularization bias} induced by the temperature floor $\alpha_{\min}$. $\qed$

Together, Theorems~\ref{thm:pairwise-kl} and~\ref{thm:disagreement_temperature_convergence} show that the disagreement-scaled soft Bellman updates approach a neighborhood of the unregularized optimum whose radius is controlled by $\alpha_{\min}$, while the policies induced by the coupled $Q$-functions remain uniformly close in KL divergence at every iteration $t$.

\section{Practical behavioral consistency via $Q$-value expectile disagreement}

Theorem~\ref{thm:pairwise-kl} suggests that we can achieve similar policies across different runs by tuning the temperature of the Boltzmann policy. 
A high temperature should lead to increased consistency across runs.
However, two crucial issues arise that we need to address. 
Theorem~\ref{thm:disagreement_temperature_convergence} relies on the difference in $Q$-values between different runs, but we do not have access to these in practice.
Therefore, we need to develop a practical estimator for $\alpha(s)$ with the information available in a {\em single} run to achieve high consistency without collapsing to a trivial solution.
In addition, high temperature introduces optimization difficulties in MaxEnt RL by amplifying off-policy extrapolation error.
We discuss how this can be mitigated by leveraging solutions from the literature on critic overestimation.

\subsection{An empirical proxy for setting \texorpdfstring{$\alpha(s)$}{alpha(s)}} 
\label{sec:empirical_de}

We now introduce \emph{$Q$-value Expectile Disagreement} (QED), which optimizes a  proxy of Theorem~\ref{thm:disagreement_temperature_convergence} that can be plugged into any MaxEnt deep RL algorithm. To obtain a practical instantiation of a behavioral-consistency regularization, we require two components: a disagreement heuristic and a practical approximation of the $\max$ operator in Equation~\ref{eq:disagreement_alpha}. 

To address the fact that each run only knows its own $Q$-function, we exploit the double $Q$ trick~\citep{fujimoto2018addressing}, where two neural network critics, $Q_{\theta_1}(s,a)$ and $Q_{\theta_2}(s,a)$, are trained in parallel from shared data. 
We use double-critic disagreement as a single-run proxy for cross-run disagreement. 
The underlying hypothesis is that the two critics are representative of the critics that would arise in two independent runs.
We provide empirical evidence for this hypothesis in Appendix~\ref{app:add_exps} by directly measuring the correlation between early double-critic disagreement and cross-run $Q$-function disagreement.
If the two critics are representative, then using its disagreement to regulate policy improvement also regulates the disagreement between policies that independent runs would induce.
In particular, when the critics disagree, the scaled temperature makes policy improvement sufficiently high-entropy that the resulting policies remain close. 
Since close policies induce similar data distributions and hence similar critic updates, this closeness persists throughout training. 

Next, we need to approximate the $\max$ over all actions in Equation~\ref{eq:disagreement_alpha}. Optimizing this maximum directly can induce an adversarial search over actions, as it highlights whichever action yields the largest critic discrepancy. 
This is undesirable, since $Q$-functions based on function approximation can be unreliable on unseen actions~\citep{thrun1993issues, fujimoto2019bcq}. Thus, the maximizer may be driven by extrapolation error rather than meaningful uncertainty, which would unnecessarily inflate $\alpha(s)$ and inject high-variance updates.
Instead, we compute an upper expectile of the double-critic difference under the current policy only. This yields a high-quantile summary of disagreement over actions the policy actually considers, capturing large but representative discrepancies without being dominated by a single worst-case action that lies far outside the support. 

In summary, at state $s$, QED draws $N$ actions $\{a_i\}_{i=1}^N$ from the current policy, $a_i \sim \pi(\cdot \mid s)$. It then evaluates both critics and aggregates the absolute differences with an expectile at level $\tau \in (0.5,1)$: 
\begin{equation*}
    \widetilde \Delta(s; Q_{\theta_1}, Q_{\theta_2}) = \operatorname{Expectile}_\tau\Bigl(\bigl\{|Q_{\theta_1}(s,a_i) - Q_{\theta_2}(s,a_i)|\bigr\}_{i=1}^N \,\big|\, a_i \sim \pi(\cdot \mid s)\Bigr).
\end{equation*}
Then, we can implement Equation~\ref{eq:disagreement_alpha} with our proxy by setting the adaptive temperature to:
\begin{equation}
    \alpha_{\text{QED}}(s)
    = \operatorname{clip}\left(
    \frac{\widetilde \Delta(s; Q_{\theta_1}, Q_{\theta_2})}{kd},
    \alpha_{\min},
    \alpha_{\max}
    \right),
\end{equation}
where $k>0$ is an algorithm-dependent hyperparameter that controls how aggressively disagreement increases temperature, analogous to $\kappa$ in Theorem~\ref{thm:pairwise-kl}, $d$ is the action dimension, used to scale constraints across tasks, and $\alpha_{\min},\alpha_{\max}$ ensure numerical stability. Note that, if the two critics agree closely at a state, $\alpha_{\text{QED}}(s)$ can become small, which can be problematic in sparse reward settings. Early in training, both critics may be near zero, causing entropy to collapse and exploration to vanish. To prevent this failure mode, we tune $\alpha_{\min} = \alpha_{\text{TE}}$ using the standard SAC temperature objective in Equation~\ref{eq:sac_alpha_loss}. $\alpha_{\max}$ is treated as a hyperparameter. 
With this, QED reduces to SAC with regular entropy tuning when the disagreement term vanishes or the divergence constraint is loose.

\subsection{Off-policy instability}

Increasing entropy typically widens the policy’s action support, which increases how often the critic is queried and bootstrapped on actions that are weakly covered by the replay distribution. To demonstrate that this is in fact an issue, we consider a toy example.

\begin{figure}[t]

    \begin{subfigure}[b]{0.99\textwidth}
        \centering
        \includegraphics[width=\linewidth, trim={0 2.5cm 0 0}, clip]{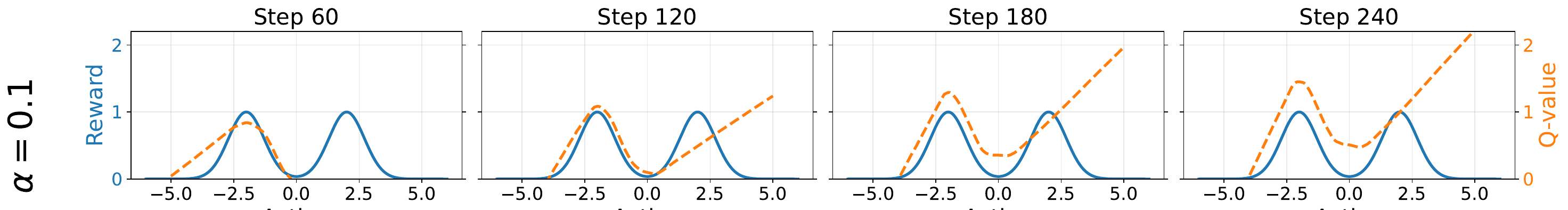}
    \end{subfigure}

    \begin{subfigure}[b]{0.99\textwidth}
        \centering
        \includegraphics[width=\linewidth, trim={0 0 0 0.9cm}, clip]{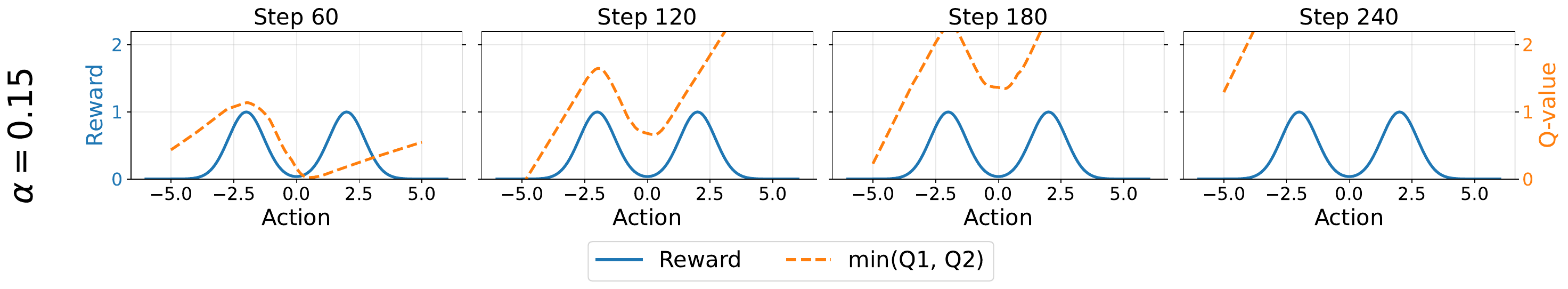}
        \caption{\small $Q$ plots demonstrate that limited action coverage induces extrapolation error outside replay support.}
    \end{subfigure}

    \begin{subfigure}[b]{0.99\textwidth}
        \centering
        \includegraphics[width=\linewidth, trim={0 2.5cm 0 0}, clip]{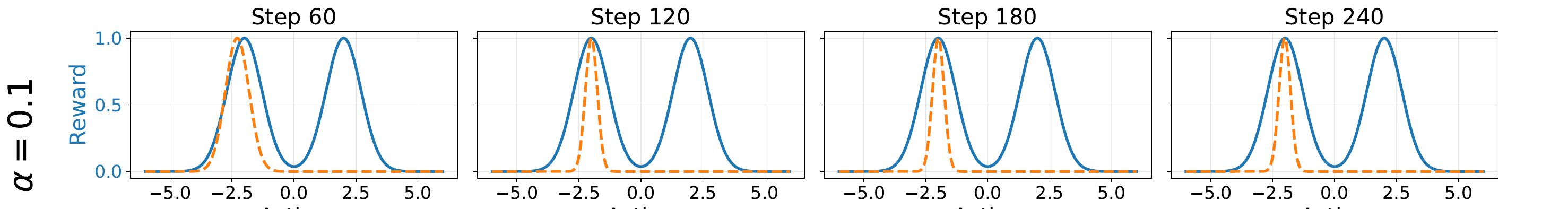}
    \end{subfigure}

    \begin{subfigure}[b]{0.99\textwidth}
        \centering
        \includegraphics[width=\linewidth, trim={0 0 0 0.9cm}, clip]{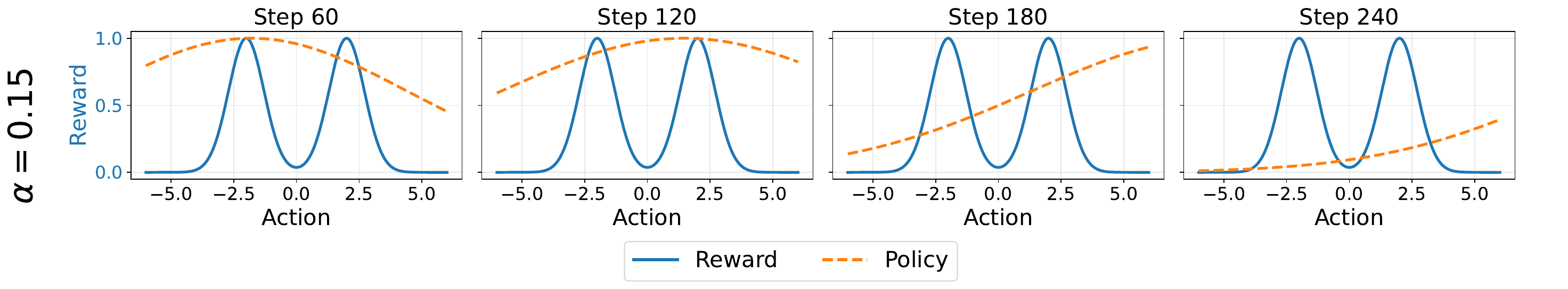}
        \caption{\small  Policy plots show that, for $\alpha=0.1$, the policy is able to consistently learn one of the modes. However, for the larger $\alpha=0.15$, the policy learns erroneous $Q$ values, pulling behavior away from the true reward modes.}
    \end{subfigure}

    \caption{\textbf{High entropy can amplify off-policy extrapolation error.} We consider the toy MDP and prefill the replay buffer with actions from only part of the action space, leaving one reward mode outside the data support. The learned $Q$-functions exhibit extrapolation error, and the policy accentuates this problem as it predicts value outside the support, particularly at larger $\alpha$ values.
    }
    \label{fig:ood_failure}
\end{figure}

\textbf{Toy example setup:}~~~
To study how increased entropy affects the behavioral consistency of an empirical algorithm with function approximation, we study an MDP with a single state, continuous action space $\actions = (-5,5)$, and no transition dynamics. 
For any action $a$, the agent receives a reward $r(a)$ given by a bimodal Gaussian density, with modes centered at $-2$ and $2$, each with variance $0.5$. 
The problem has horizon $H=2$ and discount factor $\gamma=0.9$. 
For visualization, we use a standard SAC-style implementation with a unimodal Gaussian policy over the action interval and fixed \(\alpha\), allowing us to directly compare the learned \(Q\)-landscape and the induced policy as \(\alpha\) varies.

\textbf{High entropy training collapse:}~~~
To showcase what happens during high-entropy training, we pre-fill the buffer of SAC with actions that lie only in $(-3, 0)$; thus, the right reward peak is outside of the data support. 
The agent is not allowed to collect new data to illustrate the effects of missing counterfactuals. 
We record both the critic networks' and the policy networks' predictions.

The results (Figure~\ref{fig:ood_failure}) demonstrate two key behaviors. 
First, even for small entropy coefficients ($\alpha=0.1$), the $Q$-function overestimates values for out-of-distribution (OOD) actions. 
Still, the policy learns the correct behavior because its narrow distribution prevents it from sampling and evaluating those poorly estimated OOD regions during the actor update. 
In contrast, when we increase $\alpha$ to $0.15$, the policy is forced to maximize a higher entropy bonus, resulting in an initially broad distribution that spans the entire action space (visible at Step 60). This wider sampling forces the critic to evaluate OOD actions where it lacks data, triggering the classical $Q$-value divergence spiral. 

The policy then updates to exploit these hallucinated $Q$-values rather than the true reward. Notably, the policy does not collapse into a single sharp peak over the highest $Q$-values; instead, it drifts to the right. The agent attempts to shift its probability mass toward the runaway linear extrapolation of the $Q$-values, but the strong entropy penalty forces the distribution to remain wide. Ultimately, the agent completely abandons the true reward structure, resulting in  functional policy collapse. 
This illustrates why behavioral consistency cannot be achieved by simply maximizing entropy; high entropy broadens action support, which can amplify critic extrapolation errors and destabilize learning.

\textbf{Empirical solutions:}~~~ 
The off-policy instability of high entropy RL suggests that we cannot na{\"i}vely apply the adaptive temperature scheduling to any double-critic RL algorithm and expect stable training.
We therefore evaluate the practical instantiation of our technique with two methods that are purposefully designed to handle strong action distribution shifts. 
Concretely, we use:
\begin{enumerate}[leftmargin=*, topsep=0pt]
    \item \textbf{Soft Actor-Critic with Layer Normalization (SAC-LN)}: The first algorithm we use is the traditional Soft Actor-Critic algorithm~\citep{haarnoja2018sac}. Motivated by recent work on stabilizing critic bootstrapping in off-policy reinforcement learning~\citep{hussing2024dissecting,nauman2024overestimation}, we use layer normalization~\citep{ba2016layer} in our SAC baseline. Layer normalization prohibits the Q-function from diverging when frequently queried on out-of-distribution data; however it does not fully remedy the out-of-distribution problem~\citep{hussing2024dissecting}.
    \item \textbf{Model-Augmented Data Soft-Actor Critic (MAD-SAC)}: The second algorithm is a MaxEnt version of MAD-TD~\citep{voelcker2025madtd}. MAD-TD builds on top of TD-MPC2~\citep{hansen2024tdmpc} and augments critic training with model-generated data to explicitly address the off-policy bootstrapping problem. For stability, it also uses layer normalization, as well as HL-Gauss and a self-prediction loss~\citep{li2023efficient, voelcker2024when,fujimoto2025towards}. For computational simplicity, we run the algorithm at a low update-to-data ratio and omit MPC. While this leads to minor performance loss, it is sufficient to demonstrate our claims.
\end{enumerate}

It is important to stress that, while our experimental evaluation uses these approaches, our method is agnostic to the exact base algorithm and can be applied to a wide variety of other stable off-policy RL algorithms \citep{nauman2024bigger,fujimoto2025towards,palenicek2025scaling,lee2025hyperspherical}.

\section{Experiments}
\label{sec:experiments}

We conduct experimental evaluation of QED and its ability to increase behavioral similarity across runs using the $18$ base tasks from the $\texttt{dm\_control}$ suite~\citep{tunyasuvunakool2020dmcontrol}.
The MDPs are designed with a frame skip of $2$, which is common on this benchmark~\citep{hansen2024tdmpc, voelcker2025madtd, palenicek2026xqc}; we set $\gamma=0.99$. All methods are run over $10$ random seeds, and we report $95$\% bootstrapped confidence intervals~\citep{agarwal2021deep}.

\subsection{Main experiments} 

First, we evaluate how QED reduces behavioral differences by applying it to the two algorithms described above across values of $k$. We collect $20$ evaluation rollouts with the final trained policies on all control tasks. Then, we report the cumulative return and measure similarity $\mathcal V$ using the symmetric KL divergence described in Appendix~\ref{app:kl_protocol}. The results are shown in Figure~\ref{fig:dmc_returns}. For an easier comparison, we also report the variances in Figure~\ref{fig:dmc_variance}.

\begin{figure}[t]
    \centering
    \begin{subfigure}[t]{0.66\linewidth}
        \centering
        \includegraphics[width=\linewidth]{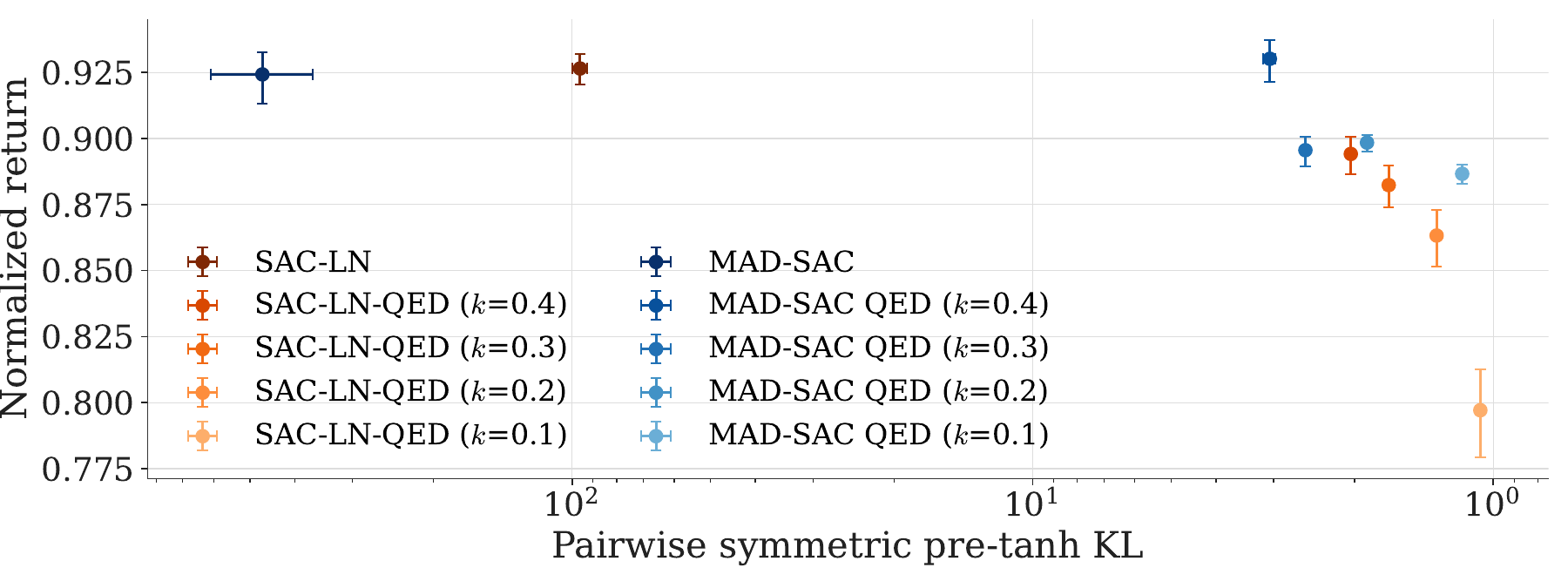}
        \caption{Symmetric KL vs returns on $\texttt{dm\_control}$}
        \label{fig:dmc_returns}
    \end{subfigure}\hfill
    \begin{subfigure}[t]{0.32\linewidth}
        \centering
        \raisebox{15mm}{
            \begin{minipage}{\linewidth}
                \centering
                \includegraphics[
                    width=\linewidth,
                    trim={0 0 0 0.5cm},
                    clip
                ]{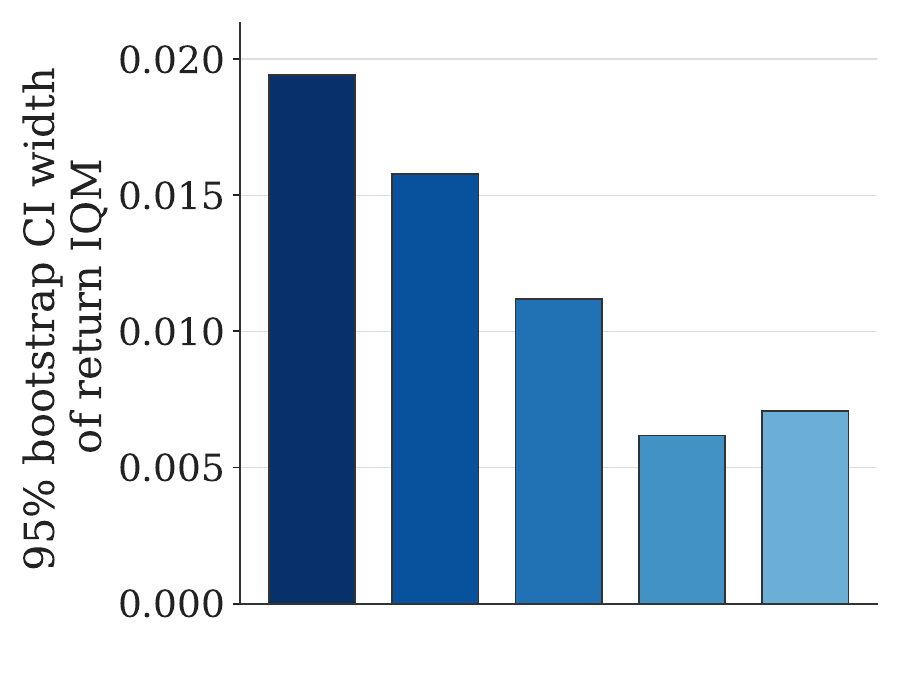}\\[-0.6em]
                \includegraphics[
                    width=4.5cm
                ]{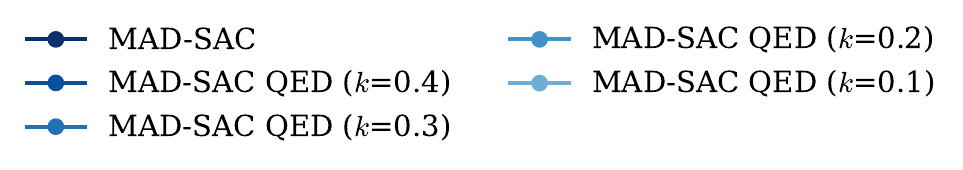}
            \end{minipage}
        }
        \caption{Variance in $\texttt{dm\_control}$}
        \label{fig:dmc_variance}
    \end{subfigure}
    \caption{\textbf{QED reduces inter-run policy divergence while preserving returns.}  (a) Final IQM of normalized return vs pairwise symmetric KL across independent training runs on the 18-task $\texttt{dm\_control}$ suite. 
    Lower KL indicates more behaviorally consistent policies. Applying QED to both SAC-LN and MAD-SAC decreases pairwise KL by about two orders of magnitude, while retaining comparable normalized return.
    (b) Width of the 95\% bootstrap CI of the evaluation IQM estimates under MAD-SAC. QED reduces return variance, suggesting that behavioral consistency implies stable returns.}
    \label{fig:main}
\end{figure}

\textbf{$\mathbf{Q}$ED increases behavioral similarity:}~~~Our first observation is that, across the board, both baseline algorithms achieve high returns; however, the KL differences between runs are very large. This makes the setting a suitable candidate for evaluating our intervention.
As we apply QED to both algorithms, we observe a drastic increase in behavioral similarity by two orders of magnitude and, with a well-chosen value of $k$, performance does not decrease. 
However, to get to very low KL divergences
(close to a value of $1$) requires small values of $k$, and as a result both algorithms see a reduction 
in return. However, MAD-SAC (which is made to deal with out-of-distribution prediction explicitly) Pareto dominates SAC as predicted and only incurs a reduction in performance of of roughly $5$\% at the lowest KL level while SAC loses more than $10$\% in performance. Appendix~\ref{app:full_exp} shows that this is in part due to a predicted loss in sample efficiency as the algorithm needs to explore until it can reduce its $Q$-function uncertainty.

\textbf{Relationship between variance and similarity:}~~~Using MAD-SAC, our results explicitly show that as $k$ decreases, the return variance shrinks as well by up to 50\%. This validates the hypothesis that behavioral similarity can be a proxy for return variance. Yet, we find that SAC itself does not exhibit this effect. This can be attributed to the fact that SAC is not designed to handle the bootstrapping problem in the same way. Runs that lose performance often do so because they become unstable, rather than because performance decreases consistently across seeds. In other words, the reduction in performance is driven by individual runs failing to learn correctly. MAD-SAC was designed to mitigate this failure mode, which is why its variance shrinks gracefully to a fixed point as $k$ decreases.

\textbf{Ablating $k$:}~~~As we increase $k$ in both algorithms, the trend is to move back toward the original performance of the base algorithms. This is again to be expected as our approach directly generalizes traditional entropy tuning. On the other hand, as we decrease $k$, the KL-divergence decreases, mapping out a Pareto frontier induced by the performance-vs-similarity trade-off. 
As the parameter $k$ controls this trade-off, we recommend starting with the largest $k$ that meaningfully reduces divergence relative to the base algorithm and lower it when consistency is prioritized over sample efficiency. In our experiments, $k \in [0.1,0.4]$ was stable for MAD-SAC, while SAC starts to slowly deteriorate in that range. Thus, $k$ is algorithm- and domain-dependent, and less-stable algorithms may need weaker constraints to avoid amplifying critic error.

\subsection{Evaluation of behavior rollout effects}

Next, we are interested in ablating the behavioral changes that we can observe due to QED to verify that we are obtaining more consistent long-term behavior and not just pointwise improvements.
To test this hypothesis, we run a quantitative and a qualitative analysis of behavior rollouts.

\begin{wrapfigure}{r}{0.45\linewidth}
    \vspace{-6pt}
    \centering

    \includegraphics[width=6.8cm]{figures/final_results/legend_action_diff_new.pdf}
    \includegraphics[width=\linewidth]{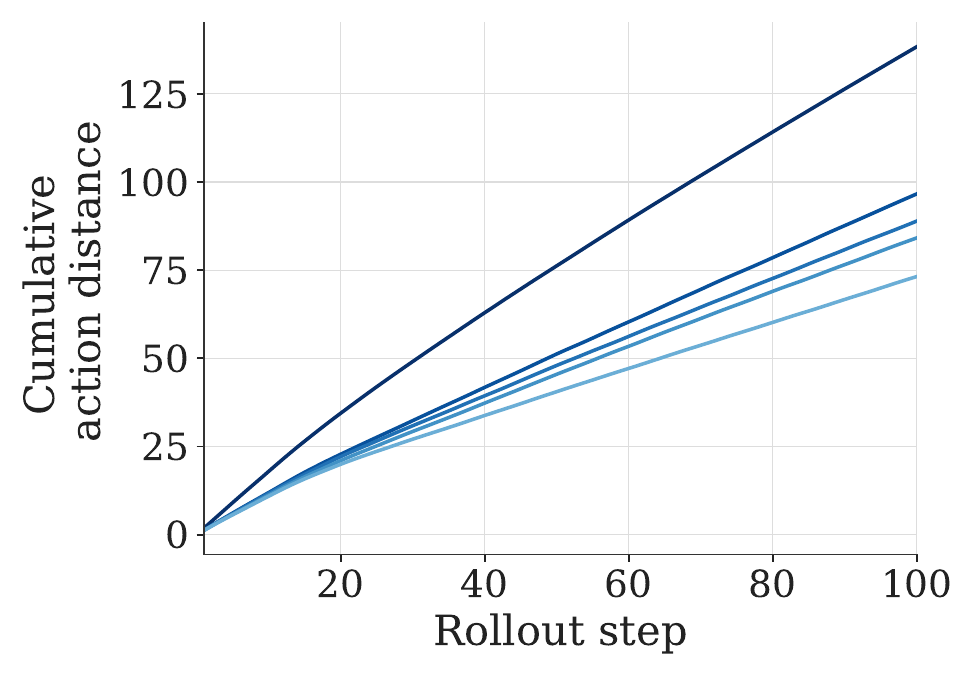}

    \caption{\textbf{QED produces more consistent rollout-level behavior across independently trained policies.} Measuring pairwise $L_2$ action distances across policies at each step, we find that QED reduces cumulative action distance.
    }
    \vspace{-20pt}
    \label{fig:action_rollouts}
\end{wrapfigure}
\textbf{Action distance:}~~~We take the trained MAD-SAC policies and roll them out for $20$ evaluation trajectories of $100$ steps. In each task, we compute the pairwise ($L_2$) action distance between
all policies trained on that task at every step. We report the cumulative distance over all steps and take the mean over all environments. Due to this evaluation's structure, it is difficult to construct a sensible measure of variance; we omit it, as the sample size is very large and confidence intervals would be tight.

Figure~\ref{fig:action_rollouts} shows that QED produces a large drop in action distance, providing quantitative evidence that lower KL is not merely due to higher-variance action distributions but rather because the \emph{distribution means} are closer in aggregate. Increasing the behavioral constraint by decreasing $k$ is correlated with a decrease in action distance. This provides  quantitative evidence that QED is learning policies that behave more similarly at test time.

\textbf{Visual similarity:}~~~Next, we examine qualitative rollouts of our policies. We choose the $\texttt{cheetah\_run}$ environment where both policies perform well and we see a reasonably low KL divergence using QED (see Appendix~\ref{app:full_exp}). We initialize environments at a fixed state. Then, we execute policies from both the target-entropy version of MAD-SAC and the QED version. We visualize the behavior of each agent as it walks for $25$ steps in Figure~\ref{fig:visualization}.

The policies trained with target entropy tuning learn behaviors that appear qualitatively distinct. One starts with a small hop, another initially walks, and a third begins with a medium-sized hop before transitioning into a larger jump.
These are three distinct behaviors obtained by regular training. 
In contrast, all policies trained with QED exhibit consistent behavioral pattern that immediately starts the task with a wide jump.
This illustrates that QED improves behavioral similarity not only quantitatively but visibly, a key requirement for making it useful for reward function tuning.

\subsection{Ablating fixed $\alpha$}

\begin{figure}[t]
    \centering
    \includegraphics[width=0.75\linewidth]{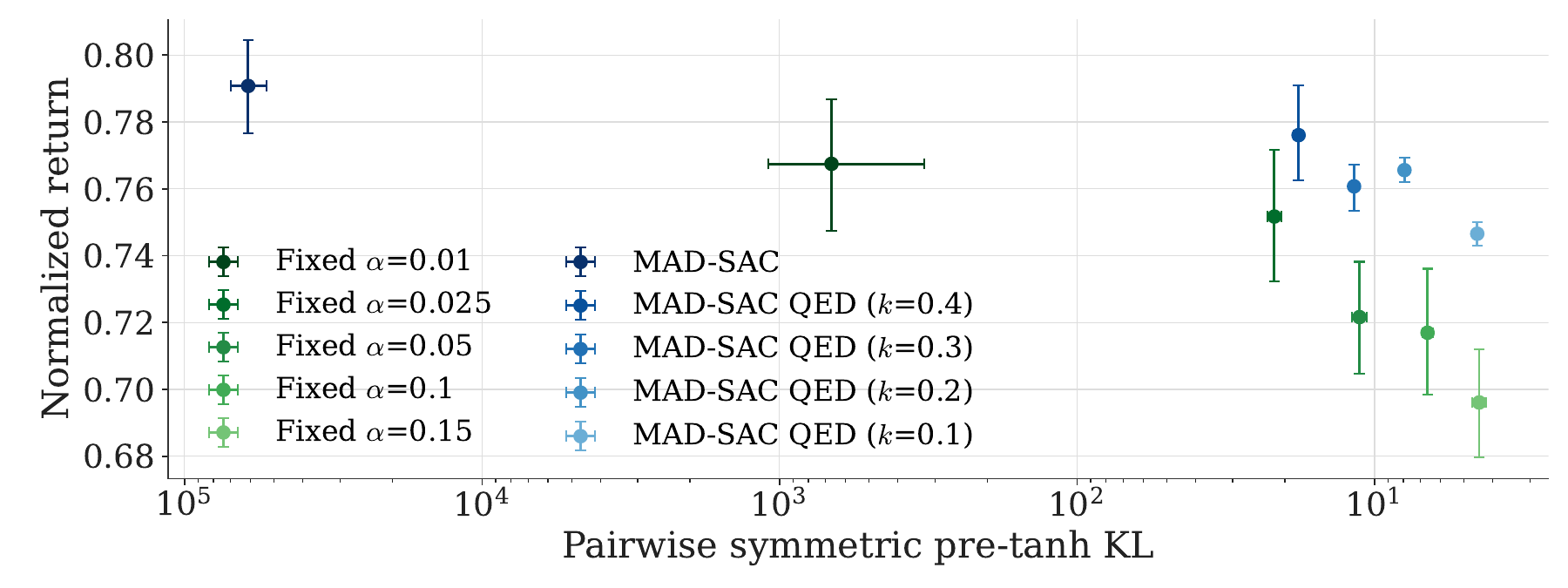}
    \caption{\textbf{Adaptively choosing $\alpha$ via QED leads to higher performance at the same KL-distance compared to static values of $\alpha$. } Mean normalized return vs pairwise symmetric KL across independent runs on the 18-task \texttt{dm\_control} suite comparing QED vs static values of $\alpha$. Fixed values of $\alpha$ lead to decrease in pairwise KL but also yield strictly lower performance than using QED.}
    \label{fig:fixedalpha}
\end{figure}

The guarantees for policy similarity without considering return are naturally also achievable by simply choosing a relatively high static $\alpha$. To illustrate the benefits of tying $\alpha$ directly to the Q-function disagreement, we compare MAD-SAC with QED to a variety of fixed alpha values. In particular, we measure the mean return and pairwise KL-divergence across policies.
Here, we report the mean rather than the IQM used previously because the mean is more sensitive to the low-return runs induced by large fixed temperatures. The results are shown in Figure~\ref{fig:fixedalpha}.

As expected, working with a large fixed $\alpha$ reduces the pairwise KL divergence between policies. However, this reduction comes with a marked decrease in mean return. In contrast, QED maintains relatively stable performance across the tested values of $k$, yielding a Pareto frontier that dominates that of fixed $\alpha$.
The decline under fixed temperatures is driven by an increasing number of low-return seeds, which lowers the mean and increases the variation across runs. 
At the lowest measured symmetric KL level, a fixed $\alpha$ incurs an approximately 7\% lower mean performance than QED.

\subsection{Tight KL control over high-variance tasks}

Finally, to demonstrate the exploration benefits of QED, we highlight a challenging task in the $\texttt{dm\_control}$ suite: the $\texttt{hopper\_hop}$ task. Across various state-of-the-art algorithms, the return variance on this task is very high and 
reported mean performance ranges from $250$-$500$ with confidence intervals often half as large as the mean returns over a few trials~\citep{doro2023barrier,lee2025hyperspherical, palenicek2026xqc}. This can be partially explained by the difficult initial state distribution of the task and the unstable dynamics of the single-legged hopper~\citep{voelcker2024can}.
Both lead to strong dissimilarity in behavior across runs.
Figure~\ref{fig:hopper} shows the reward curves during training of this task.
Relatedly, \citet{hussing2024dissecting} argue that this task likely requires strong exploration.

\begin{wrapfigure}{r}{0.45\linewidth}
    \vspace{-6pt}
    \centering
        \centering

        \includegraphics[width=4.8cm]{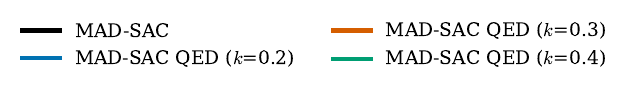}
        \includegraphics[width=\linewidth]{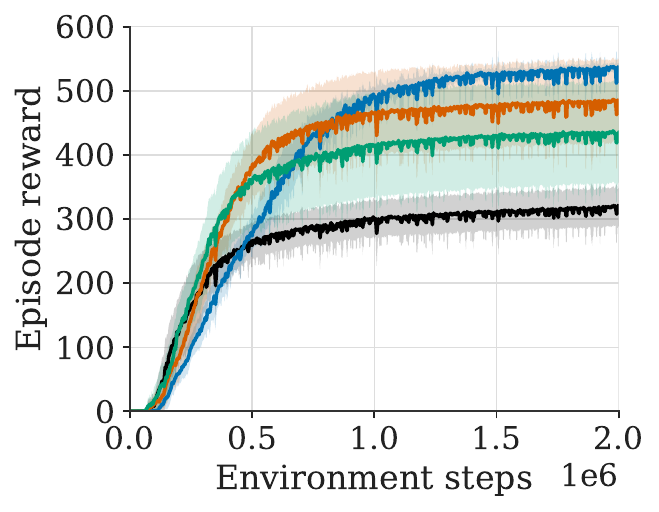}
        
        \caption{\textbf{QED reduces training-time variance and policy divergence on a high-variance control task.} Return over training steps. QED improves performance and reduces dispersion across seeds. 
        }
        \label{fig:hopper}
        \vspace{-40pt}
\end{wrapfigure}
As expected, the base MAD-SAC algorithm has high variance 
during training 
but adding QED improves performance.
A complementary interpretation of QED is that it increases policy entropy and thus exploration.
An interesting phenomenon occurs when the QED constraint is strong enough at $k=0.2$, as all runs converge to an identical performance.
This suggests that QED shapes exploration, with stronger constraints aligning executions in policy space.

\subsection{Limitations} \label{sec:limitations}

A key limitation of QED is that it relies on pointwise $Q$-estimates and their agreement. In these experiments, we find that disagreement can often take a long time to vanish, and it is not currently clear when vanishing disagreement is an appropriate criterion in practice. 
Furthermore, while QED provides a consistent boost in behavioral similarity, higher-dimensional tasks reveal a more pronounced trade-off between performance and similarity. KL values can often remain above $1$, which we find to be an important threshold for multi-step similarity. We hope that future work can improve these issues 
while ensuring consistent behavior over very long horizons.

\section{Related work}

\label{sec:related}

{\bf Reproducibility, replicability, and run-to-run variability in RL.~~}
Deep RL is well known to be sensitive to random seeds, implementation choices, environment stochasticity, and evaluation protocol~\citep{islam2017reproducibility,henderson2018matters,colas2018how}. 
This sensitivity makes empirical conclusions statistically fragile, especially when comparisons rely on small numbers of runs, motivating robust aggregate metrics such as the interquartile mean and bootstrap confidence intervals~\citep{agarwal2021deep}. 
At the same time, benchmarking return alone can be limited as a scientific instrument for understanding the behavior learned by RL algorithms~\citep{jordan2024position}.  

A related theoretical line studies replicability~\citep{impagliazzo2022reproducibility}, asking when repeated executions of a learning algorithm produce identical outputs. 
Recent work has begun to formalize this question for RL~\citep{eaton2023replicable,karbasi2023replicability,hopkins2025from,eaton2026replicable}, but existing results primarily focus on tabular or linear-function-approximation settings. 
\citep{eaton2026replicable} provide a first approach to minimize disagreement across independent runs in neural networks inspired by their theoretical results but their experimental results are limited to two environments and discrete action spaces.
Closer to ours, empirical work has studied variability beyond aggregate return, including performance variation among Atari agents~\citep{clary2018variability} and reliability metrics for learned policies~\citep{chan2020measuring}. Policy churn during deep RL training has also been analyzed~\citep{schaul2022policy} and mitigated~\citep{tang2024improving}, whereas we directly target cross-run behavioral consistency.

{\bf Off-policy error~~}
Our analysis is also connected to the broader literature on instability in off-policy RL with function approximation. 
The classical ``deadly triad'' of bootstrapping, off-policy data, and function approximation has long been understood as a source of divergence~\citep{thrun1993issues,tsitsiklis1996analysis,precup2001off,hasselt2018deep}.
One of the most popular interventions in modern actor-critic algorithms to deal with overestimation is the double critic estimator~\citep{hasselt2010double,hasselt2016deep,fujimoto2018addressing}.
This idea has been pushed to the limit in ensemble critics that help mitigate divergence~\citep{chen2021randomized, hiraoka2022dropout}.
Alternative interventions that stabilize learning include conservative value updates~\citep{fujimoto2019bcq, kumar2020cql}, categorical value function representations~\citep{bellemare2017distributional, farebrother2024stop} and normalization~\citep{ball2023efficient, bhatt2024crossq, hussing2024dissecting, nauman2024overestimation, lyle2025disentangling}.
Recent work has shown that high update-to-data ratios can improve sample efficiency~\citep{nikishin2022primacy, doro2023barrier} but exacerbate value divergence~\citep{hussing2024dissecting} unless stabilized by regularization~\citep{palenicek2025scaling}, model-augmented data~\citep{voelcker2025madtd}, architectural changes~\citep{hussing2024dissecting, fujimoto2025towards,palenicek2026xqc}, or algorithmic interventions~\citep{romeo2025speq}. 
Our toy examples show that high entropy can interact with these same off-policy errors by widening policy support and forcing actor updates to query poorly supported regions of the critic. 
QED includes a built-in mechanism to avoid querying actions with spuriously high Q-value disagreement, but it is still most effective when paired with actor-critic methods that already control overestimation well.

{\bf Policy regularization in RL.~~}
MaxEnt RL augments reward maximization with an entropy bonus, producing stochastic policies that encourage exploration and avoid premature collapse to deterministic actions~\citep{ziebart2008maximum,haarnoja2017reinforcement,haarnoja2018sac}
This framework is closely connected to KL-regularized control and probabilistic inference views of RL, where policy improvement is regularized toward a prior or reference distribution~\citep{todorov2009efficient,peters2010relative,levine2018reinforcement, geist2019atheory}. 
Soft Actor-Critic instantiates this idea in deep off-policy actor-critic learning, using entropy regularization and automatic temperature tuning to target a desired policy entropy~\citep{haarnoja2018sac, haarnoja2018applications}. 
Other forms of regularization, including mutual information~\citep{leibfried2020mutual}, control priors~\citep{johannink2019residual}, and policy regularization~\citep{fox2016taming,liu2021regularization}, have been used to improve exploration, robustness, and optimization stability. Closest in spirit, ~\citet{cheng2019control} regularize policies toward a control-theoretic prior to reduce variance in reinforcement learning. In contrast, we reduce cross-run behavioral divergence by adapting the entropy temperature from critic disagreement, without requiring an external controller prior.

{\bf Uncertainty-driven exploration and critic disagreement.~~}
A large body of work uses uncertainty estimates to guide exploration or reduce value-estimation error, including ensembles, bootstrapping, optimistic objectives, and epistemic bonuses~\citep{osband2016deep, donoghue2018the, osband2019deepexplorationrandomizedvalue,kamil2019optimistic,chen2021randomized,lee2021sunrise,moskovitz2021tactical}. 
QED also uses critic disagreement, but for a different role, treating disagreement as a local proxy for policy instability, with increased entropy implicitly promoting exploration in uncertain regions.

\section{Conclusion}
\label{sec:conclusion}

As RL moves toward larger models, longer training horizons, and real-world deployment, variability across runs becomes more than an evaluation nuisance: it limits our ability to predict, debug, and systematically improve learning systems. Behavior-consistent RL offers a framework to address this problem, and we show in this paper that enforcing consistent policies is both feasible and beneficial. Our results suggest that behavioral consistency may be a missing ingredient in scalable and reliable RL. If independent runs of the same RL algorithm can converge to different behavior, then evaluation, reward design, safety validation, and scaling-law estimation of performance all become more difficult. 

Our results also expose new research questions. Many tasks admit multiple valid behavior strategies, raising the question of when diversity across runs should be preserved versus eliminated. Similarly, the tradeoff between consistency and sample complexity suggests deeper connections between uncertainty, replicability, and scalable policy optimization.  Controlling behavioral variability opens a new direction for RL research, enabling the design of algorithms whose results are not only performant but also stable, predictable, and scientifically reproducible.

\section*{Acknowledgments}

MH and EE were partially supported by DARPA grant \#HR00112420305. LdA was supported by a First-Year Fellowship from the Princeton University Graduate School.
The authors are also pleased to acknowledge that the work reported on in this paper was substantially performed using the General Robotics, Automation, Sensing and Perception Lab's cluster at UPenn as well as the Princeton Research Computing resources at Princeton University which is consortium of groups led by the Princeton Institute for Computational Science and Engineering (PICSciE) and Office of Information Technology's Research Computing.

Our thanks to Raj Ghugare and Cathy Ji for reading early drafts of this work. Special thanks to the Graduate Student Group at Frist Center and Sertraline. Any opinions, findings, and conclusions or recommendations expressed in this material are those of the author(s) and do not necessarily reflect the views, position, or policy of DARPA or the US Government.

\bibliographystyle{plainnat}
\bibliography{references}

\newpage

\appendix

\section{Proofs}

We provide the proofs for the two theoretical results, Theorems~\ref{thm:pairwise-kl} and~\ref{thm:disagreement_temperature_convergence}, that motivate QED. We prove that, for two Boltzmann policies induced by different $Q$-functions at a fixed state, setting the shared temperature proportional to their pointwise $Q$-disagreement bounds the KL divergence between the resulting policies. We then prove that using this disagreement-scaled temperature inside soft-value iteration still approaches a neighborhood of the unregularized optimum: the only persistent error comes from the entropy floor $\alpha_{\min}$, while the disagreement-dependent term vanishes over time. Together, these arguments formalize the principle behind QED: entropy should be high when value estimates disagree, but should shrink as those disagreements contract.

\subsection{Proof of Theorem~\ref{thm:pairwise-kl}} \label{app:kl_proof}

\begin{theorem*}[Pairwise KL control via disagreement-scaled temperature]
Assume $\actions$ is finite and fix a state $s\in\states$.
Let $Q^{(1)}(s,\cdot),Q^{(2)}(s,\cdot)\in\mathbb R^{|\actions|}$ be two action-value vectors, and fix $\kappa>0$ and $\alpha_{\min} > 0$.
Define the shared temperature
\begin{equation} 
\alpha(s) \doteq \max\left\{\alpha_{\min}, \frac{\|Q^{(1)}(s,\cdot)-Q^{(2)}(s,\cdot)\|_\infty}{\kappa}\right\}.
\end{equation}
Let $\pi^{(1)}(\cdot\mid s)$ and $\pi^{(2)}(\cdot\mid s)$ be the Boltzmann policies at temperature $\alpha(s)$. Then
\begin{equation*}
D_{\mathrm{KL}}(\pi^{(1)}(\cdot\mid s)\|\pi^{(2)}(\cdot\mid s))\le 2\kappa.
\end{equation*}
\end{theorem*}

\begin{proof}
Fix some $s \in S$. Let $Z^{(i)}\doteq\sum_{b\in\actions}\exp(Q^{(i)}(s,b)/\alpha(s)), \, i\in\{1,2\}.$
For any action $a$, we can write out the log ratio between our two policies as
\begin{equation*}
\log\frac{\pi^{(1)}(a\mid s)}{\pi^{(2)}(a\mid s)}
= \frac{Q^{(1)}(s,a)-Q^{(2)}(s,a)}{\alpha(s)} + \log\frac{Z^{(2)}}{Z^{(1)}}.
\end{equation*}
Taking expectation with respect to $a\sim\pi^{(1)}(\cdot\mid s)$ gives the KL divergence 
\begin{equation*}
D_{\mathrm{KL}}(\pi^{(1)}\|\pi^{(2)})
= \frac{\mathbb E_{a\sim\pi^{(1)}}[Q^{(1)}(s,a)-Q^{(2)}(s,a)]}{\alpha(s)}
+ \log\frac{Z^{(2)}}{Z^{(1)}}.
\end{equation*}

We bound these two terms independently. For the first term, it is easy to see that
\begin{equation*}
\frac{\mathbb E_{a\sim\pi^{(1)}}[Q^{(1)}(s,a)-Q^{(2)}(s,a)]}{\alpha(s)}
\le \frac{\mathbb E_{a\sim\pi^{(1)}}[|Q^{(1)}(s,a)-Q^{(2)}(s,a)|]}{\alpha(s)}
\le \frac{\|Q^{(1)}(s,\cdot)-Q^{(2)}(s,\cdot)\|_\infty}{\alpha(s)}.
\end{equation*}

For the second term, for every $b\in\actions$,
\begin{align*}
    \exp(Q^{(2)}(s,b)/\alpha(s))
    &= \exp((Q^{(1)}(s,b)-(Q^{(1)}(s,b)-Q^{(2)}(s,b)))/\alpha(s)) \\
    &\le \exp(Q^{(1)}(s,b)/\alpha(s))
    \exp(\|Q^{(1)}(s,\cdot)-Q^{(2)}(s,\cdot)\|_\infty/\alpha(s)).
\end{align*}
Summing over $b$ yields
$Z^{(2)}\le Z^{(1)}\exp(\|Q^{(1)}(s,\cdot)-Q^{(2)}(s,\cdot)\|_\infty/\alpha(s)),$
and therefore
\begin{equation*}
\log\frac{Z^{(2)}}{Z^{(1)}}
\le \frac{\|Q^{(1)}(s,\cdot)-Q^{(2)}(s,\cdot)\|_\infty}{\alpha(s)}.
\end{equation*}

Combining the bounds,
\begin{equation*}
D_{\mathrm{KL}}(\pi^{(1)}\|\pi^{(2)})
\le \frac{2\|Q^{(1)}(s,\cdot)-Q^{(2)}(s,\cdot)\|_\infty}{\alpha(s)}.
\end{equation*}
The result follows from the construction of $\alpha(s)$.
\end{proof}

\paragraph{Symmetric KL.}
Although Theorem~\ref{thm:pairwise-kl} is stated for directed KL, the same
bound immediately applies to the symmetric KL used in our empirical metric.
Indeed, applying the theorem once to
\(D_{\mathrm{KL}}(\pi^{(1)}\|\pi^{(2)})\) and once with the two \(Q\)-functions
reversed gives
\[
D_{\mathrm{KL}}(\pi^{(1)}\|\pi^{(2)})\le 2\kappa,
\qquad
D_{\mathrm{KL}}(\pi^{(2)}\|\pi^{(1)})\le 2\kappa.
\]
Therefore,
\[
\frac{1}{2}
\left[
D_{\mathrm{KL}}(\pi^{(1)}\|\pi^{(2)})
+
D_{\mathrm{KL}}(\pi^{(2)}\|\pi^{(1)})
\right]
\le 2\kappa.
\]
Thus the idealized Boltzmann analysis controls the same symmetrized divergence form used in the experiments.

\subsection{Proof of Theorem~\ref{thm:disagreement_temperature_convergence}} \label{app:convergence_proof}

\begin{theorem*}[Convergence under disagreement-scaled temperature]
Assume $\alpha_{\min}>0$. Let $Q^\star$ denote the optimal Q-function of the unregularized MDP. For a temperature $\alpha:\states\to\mathbb R_{>0}$, define
\begin{equation*}
(\mathcal T_\alpha Q)(s,a)
\doteq
r(s,a)
+
\gamma\mathbb E_{s'\sim P(\cdot\mid s,a)}
\left[
\alpha(s')\log\sum_{a'\in\actions}
\exp\left(\frac{Q(s',a')}{\alpha(s')}\right)
\right].
\end{equation*}
Consider the coupled iterates $Q_{t+1}^{(i)}=\mathcal T_{\alpha_t}Q_t^{(i)}$, $i\in\{1,2\}$, where $\alpha_t(s)$ is the shared temperature from \eqref{eq:disagreement_alpha}. Let $\Delta_0\doteq \|Q_0^{(1)}-Q_0^{(2)}\|_\infty$ be the initial disagreement. Then, for $i\in\{1,2\}$ and $t\ge 0$,
\begin{equation*}
\|Q_t^{(i)}-Q^\star\|_\infty
\le
\gamma^t\|Q_0^{(i)}-Q^\star\|_\infty
+
\frac{\gamma\alpha_{\min}\log|\actions|}{1-\gamma}(1-\gamma^t)
+
\frac{\Delta_0\log|\actions|}{\kappa}\,t\gamma^t .
\end{equation*}
\end{theorem*}

\begin{proof}
    Let $\mathcal T$ denote the unregularized Bellman optimality operator and let $e_t^{(i)}\doteq \|Q_t^{(i)}-Q^\star\|_\infty$. We use two standard facts about soft Bellman backups from maximum-entropy value iteration~\citep{ziebart2008maximum, haarnoja2018sac,levine2018reinforcement}. First, for any fixed positive temperature $\alpha$, $\mathcal T_\alpha$ is a $\gamma$-contraction in $Q$. Second, the soft value is a one-sided approximation to the hard maximum:
    \begin{equation}
    0
    \le
    (\mathcal T_\alpha Q)(s,a)-(\mathcal T Q)(s,a)
    \le
    \gamma\log|\actions|\,
    \mathbb E_{s'\sim P(\cdot\mid s,a)}[\alpha(s')] .
    \end{equation}
    Hence,
    \begin{equation}
    \|\mathcal T_{\alpha_t}Q-\mathcal T Q\|_\infty
    \le
    \gamma\log|\actions|\,\|\alpha_t\|_\infty .
    \end{equation}
    
    The schedule-specific step is to bound $\|\alpha_t\|_\infty$. Since both runs use the same $\alpha_t$, the standard contraction argument applies conditionally on $\alpha_t$:
    \begin{equation}
    \|Q_{t+1}^{(1)}-Q_{t+1}^{(2)}\|_\infty
    =
    \|\mathcal T_{\alpha_t}Q_t^{(1)}-\mathcal T_{\alpha_t}Q_t^{(2)}\|_\infty
    \le
    \gamma\|Q_t^{(1)}-Q_t^{(2)}\|_\infty .
    \end{equation}
    Iterating gives
    \begin{equation}
    \|Q_t^{(1)}-Q_t^{(2)}\|_\infty
    \le
    \gamma^t\Delta_0 .
    \end{equation}
    Therefore, by the disagreement-scaled definition of $\alpha_t$,
    \begin{equation}
    \|\alpha_t\|_\infty
    \le
    \alpha_{\min}
    +
    \frac{\|Q_t^{(1)}-Q_t^{(2)}\|_\infty}{\kappa}
    \le
    \alpha_{\min}
    +
    \frac{\gamma^t\Delta_0}{\kappa}.
    \end{equation}
    
    Now compare each iterate to $Q^\star$. Since $\mathcal TQ^\star=Q^\star$ and $\mathcal T$ is a $\gamma$-contraction,
    \begin{equation}
    \begin{aligned}
    e_{t+1}^{(i)}
    &=
    \|\mathcal T_{\alpha_t}Q_t^{(i)}-\mathcal TQ^\star\|_\infty \\
    &\le
    \|\mathcal T Q_t^{(i)}-\mathcal TQ^\star\|_\infty
    +
    \|\mathcal T_{\alpha_t}Q_t^{(i)}-\mathcal TQ_t^{(i)}\|_\infty \\
    &\le
    \gamma e_t^{(i)}
    +
    \gamma\log|\actions|
    \left(
    \alpha_{\min}
    +
    \frac{\gamma^t\Delta_0}{\kappa}
    \right).
    \end{aligned}
    \end{equation}
    Unrolling this approximate-contraction recursion gives
    \begin{equation}
    \begin{aligned}
    e_t^{(i)}
    &\le
    \gamma^t e_0^{(i)}
    +
    \gamma\alpha_{\min}\log|\actions|
    \sum_{j=0}^{t-1}\gamma^{t-1-j}
    +
    \frac{\gamma \Delta_0\log|\actions|}{\kappa}
    \sum_{j=0}^{t-1}\gamma^{t-1-j}\gamma^j \\
    &=
    \gamma^t e_0^{(i)}
    +
    \frac{\gamma\alpha_{\min}\log|\actions|}{1-\gamma}(1-\gamma^t)
    +
    \frac{\Delta_0\log|\actions|}{\kappa}\,t\gamma^t .
    \end{aligned}
    \end{equation}
    Substituting back $e_t^{(i)}=\|Q_t^{(i)}-Q^\star\|_\infty$ proves the stated bound.
\end{proof}

\section{Hyperparameters}
\label{app:hparams}

All experiments use the $18$ \texttt{dm\_control} tasks listed in Table~\ref{tab:dmc_tasks}. 
We use the same task set, training budget, evaluation protocol, and random seeds for SAC-LN, MAD-SAC, and their QED variants. 
For our implementation, we use the original MAD-TD codebase. The only three hyperparameters that we change are: proportion real to $0.5$ (for MAD-SAC), no actor layer normalization, and gradient clipping set to $20$. 
The shared experimental settings are summarized in Table~\ref{tab:exp_settings}. 
All environments use frame skip $2$ and discount factor $\gamma=0.99$. 
Each method is trained with $10$ independent random seeds for $2$M environment steps. 
Final performance is computed from $20$ evaluation rollouts per trained policy, and aggregate results are reported using IQM with $95\%$ bootstrap confidence intervals.

\begin{table}[h]
    \centering
    \small
    \caption{The $18$ \texttt{dm\_control} tasks used in the main experiments.}
    \label{tab:dmc_tasks}
    \begin{tabular}{llll}
        \toprule
        \multicolumn{4}{c}{\textbf{Tasks}} \\
        \midrule
        \texttt{acrobot\_swingup} &
        \texttt{cartpole\_balance} &
        \texttt{cartpole\_swingup} &
        \texttt{cheetah\_run} \\
        \texttt{finger\_spin} &
        \texttt{finger\_turn\_easy} &
        \texttt{finger\_turn\_hard} &
        \texttt{fish\_swim} \\
        \texttt{hopper\_hop} &
        \texttt{hopper\_stand} &
        \texttt{pendulum\_swingup} &
        \texttt{quadruped\_run} \\
        \texttt{quadruped\_walk} &
        \texttt{reacher\_easy} &
        \texttt{reacher\_hard} &
        \texttt{walker\_run} \\
        \texttt{walker\_stand} &
        \texttt{walker\_walk} &
        & \\
        \bottomrule
    \end{tabular}
\end{table}

\begin{table}[h]
    \centering
    \small
    \caption{Core experimental settings used across the main \texttt{dm\_control} experiments.}
    \label{tab:exp_settings}
    \begin{tabular}{ll}
        \toprule
        Setting & Value \\
        \midrule
        Number of tasks & $18$ \\
        Number of seeds & $10$ \\
        Training budget & $2$M environment steps \\
        Frame skip & $2$ \\
        Discount factor & $\gamma=0.99$ \\
        Final evaluation & $20$ rollouts per trained policy \\
        Aggregate return metric & IQM with $95\%$ bootstrap CI \\
        Policy-divergence metric & Pairwise symmetric pre-tanh KL \\
        Rollout behavior metric & Cumulative pairwise $\ell_2$ action distance \\
        Action-distance rollout length & $100$ steps \\
        \bottomrule
    \end{tabular}
\end{table}

For the baseline methods, we evaluate SAC-LN and MAD-SAC with standard target-entropy tuning. 
SAC-LN uses a SAC-style actor-critic with double critics and layer normalization. 
MAD-SAC uses the same maximum-entropy actor-critic structure, but augments critic training with model-generated data following the MAD-TD family of methods. 
For computational simplicity, MAD-SAC is evaluated without MPC at action-selection time.

For QED variants, we replace the scalar entropy temperature with the state-dependent disagreement-scaled temperature described in Section~\ref{sec:empirical_de}. 
The target-entropy temperature $\alpha_{\rm TE}$ is retained as a learned temperature floor, so QED reduces to ordinary target-entropy tuning when the disagreement term is inactive. 
We compute disagreement using $N=8$ actions sampled from the current policy and aggregate the double-critic differences with expectile level $\tau=0.9$. 
The disagreement is normalized by the action dimension, and the maximum temperature is clipped at $\alpha_{\max}=0.2$.

\section{Compute resources}
\label{app:compute}

All experiments were run on a shared academic GPU cluster. 
Each training run used a single GPU and 8 CPU cores, with 80G of system memory requested per job. 
The main $\texttt{dm\_control}$ experiments consist of $18$ tasks, $10$ random seeds, and the method variants listed with all seeds run simultaneously on the same GPU and trained for $2$M environment steps. 
Wall-clock time varied by task and method, with typical runs taking approximately \texttt{8}--\texttt{48} hours. 
The total reported experimental suite required approximately 4,000 GPU-hours, excluding preliminary debugging and failed runs.

For reproducibility, the reported experiments use publicly available simulated environments from $\texttt{dm\_control}$ and do not require specialized hardware beyond standard CUDA-enabled GPUs. 
The QED variants add only a small amount of computation relative to their corresponding base algorithms: at each actor update, QED samples $N=8$ actions from the current policy and evaluates both critics to compute the expectile disagreement statistic. 
This overhead is minor compared with the cost of the usual actor--critic updates and model-augmented critic training in MAD-SAC.

\section{Additional experiments} \label{app:add_exps}

QED uses disagreement between two critics in a single training run as a practical proxy for the disagreement that would be observed between independently trained runs. In particular, we argue that if the disagreement between the two critics is representative of cross-run disagreement initially, the distribution from which we sample our data stays continuously close as we ensure our policies stay close. We evaluate this hypothesis directly by measuring whether early within-run critic disagreement predicts cross-seed $Q$-function disagreement across our \texttt{dm\_control} 18 task suite.

For each task, we take the MAD-SAC runs (10 seeds) trained with standard target-entropy tuning and evaluate the critics over the early-training interval (first 15 rollouts). For each rollout, we compute the mean difference between the double critic over a sampled batch and compare it to the pairwise difference in mean Q-values of the first Q-function across all seeds. Figure~\ref{fig:critic-proxy} shows a strong positive relationship between the within-run double-critic disagreement and the cross-seed $Q$-function disagreement across tasks. This suggests that tasks where the two critics disagree (early in training) strongly correlates ($R^2=0.77$) with evaluation $Q_1$ dispersion after training, which supports our use of in-run $Q$-function disagreement as a practical proxy for cross-seed variability. 

\begin{figure}[H]
    \centering
    \includegraphics[width=0.5\linewidth]{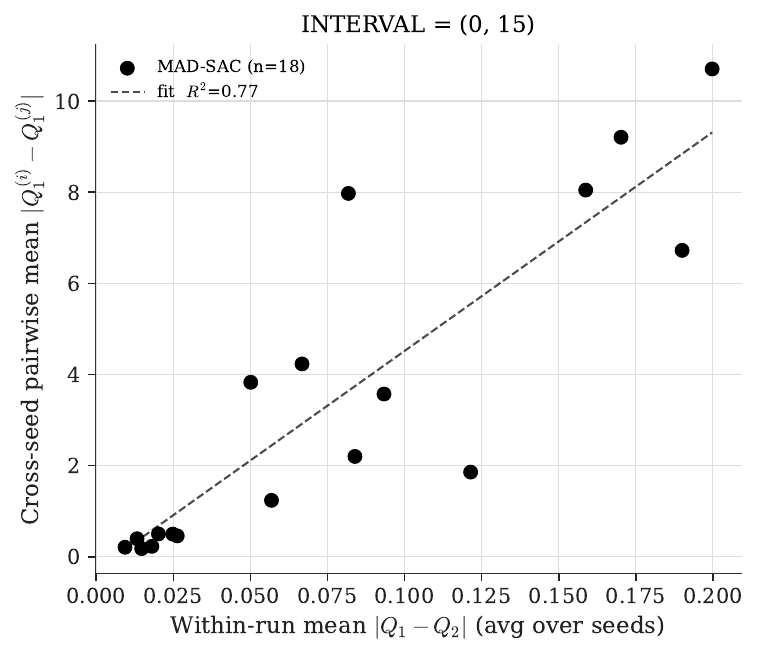}
    \caption{Early double-critic disagreement predicts cross-seed $Q$-function disagreement.}
    \label{fig:critic-proxy}
\end{figure}

\section{Policy-divergence evaluation protocol}
\label{app:kl_protocol}

Definition~\ref{def:behavior-consistent} defines inter-run variability \(\mathcal V\) using a generic divergence \(D\) between policies evaluated under a fixed protocol. In our experiments, we instantiate \(D\) as the symmetric KL divergence between the pre-tanh Gaussian action distributions produced by the actor.

\paragraph{On-policy inter-run evaluation.}
For each task, method, and checkpoint, we construct an evaluation state set
\(\mathcal S_{\mathrm{eval}}^{m,t}\) from evaluation rollouts generated by the \(10\) independently trained seeds of method \(m\) at checkpoint \(t\). 
We then evaluate all pairwise policy divergences among those \(10\) seeds on this shared within-method state set. 
For a fixed task, method, and checkpoint, every policy pair is compared on the same collection of states, and the resulting quantity measures variability across independent executions of the same learning algorithm.

This protocol is intentionally on-policy. 
Behavior-consistency concerns the stability of the behaviors that a learning algorithm actually produces under its own training and evaluation distribution. 
Evaluating each method on the states induced by its own independently trained policies measures the realized behavioral variability of that method: whether different random seeds lead to different action distributions in the regions of the state space that the method reaches at evaluation time.

Given two policies \(\pi_i\) and \(\pi_j\) from method \(m\), let their
pre-tanh Gaussian distributions at state \(s\) be
\[
    \tilde{\pi}_i(\cdot\mid s)
    =
    \mathcal N(\mu_i(s), \operatorname{diag}(\sigma_i^2(s))),
    \qquad
    \tilde{\pi}_j(\cdot\mid s)
    =
    \mathcal N(\mu_j(s), \operatorname{diag}(\sigma_j^2(s))).
\]
We define the pairwise policy divergence as
\[
    D_m(\pi_i,\pi_j)
    =
    \frac{1}{|\mathcal S_{\mathrm{eval}}^{m,t}|}
    \sum_{s\in\mathcal S_{\mathrm{eval}}^{m,t}}
    \frac{1}{2}
    \left[
    D_{\mathrm{KL}}\bigl(\tilde{\pi}_i(\cdot\mid s)\,\|\,\tilde{\pi}_j(\cdot\mid s)\bigr)
    +
    D_{\mathrm{KL}}\bigl(\tilde{\pi}_j(\cdot\mid s)\,\|\,\tilde{\pi}_i(\cdot\mid s)\bigr)
    \right].
\]
The inter-run variability of method \(m\) at checkpoint \(t\) is then
\[
    \mathcal V_m(t)
    =
    \frac{1}{I(I-1)}
    \sum_{i\neq j}
    D_m(\pi_i,\pi_j),
\]
where \(I=10\) is the number of independent training seeds.

\paragraph{Why evaluate on each method's induced states?}
A fixed exogenous state set can be useful for pointwise policy comparison, but it does not directly measure the behavioral variability realized by an RL algorithm at deployment time. 
In continuous-control tasks, the states visited by a policy are part of the learned behavior: different gaits, balance strategies, or recovery maneuvers induce different state distributions. 
Since our goal is to measure whether independent runs of the same method produce consistent behaviors, we evaluate each method on the states reached by that method's own evaluation rollouts. 
This gives a method-level behavioral variability estimate that includes the regions of the state space actually occupied by the learned policies.

This choice also avoids evaluating policies primarily on states that are irrelevant or rarely visited under that method's learned behavior. 
For example, a policy may be highly variable on off-distribution states but consistent on the trajectory manifold it actually follows; conversely, a method may produce different trajectory manifolds across seeds, which will be reflected in the states collected from its own rollouts and in the resulting inter-run KL. 
The
metric therefore captures realized behavioral consistency rather than pointwise discrepancy detached from the method's induced occupancy measure.

\paragraph{Cross-method interpretation.}
When comparing methods, \(\mathcal V_m(t)\) should be interpreted as each method's on-policy inter-run variability: the average pairwise divergence among independent runs of that method on the states induced by that method. 
It does not require all methods to be evaluated on an identical exogenous state set, because the object of comparison is the variability of each learning procedure under its own induced evaluation distribution.

\paragraph{Relation to the Boltzmann-policy analysis.}
Theorem~\ref{thm:pairwise-kl} and our discussion of symmetric KL show that, for exact Boltzmann policy improvement in finite action spaces, disagreement-scaled temperature controls both directed and symmetric KL between the induced policies.
Our continuous-control experiments use tanh-squashed Gaussian actors, so the empirical divergence is computed between the actual actor distributions used by SAC-LN and MAD-SAC. 
We compute this KL in the pre-tanh Gaussian space for numerical stability.
Thus, the empirical metric is the natural continuous-control analogue of the symmetrized policy divergence controlled in the idealized Boltzmann setting: both measure how sensitive the policy-improvement step is to run-specific value estimates, but they are instantiated for different policy classes.

\section{Full experimental results} \label{app:full_exp}

We report the per-task reward and policy-divergence curves underlying the aggregate results in Section~\ref{sec:experiments}. 
Figures~\ref{fig:reward_grid_sac} and~\ref{fig:reward_grid_madsac} show how QED affects learning performance across all $18$ \texttt{dm\_control} tasks for SAC-LN and MAD-SAC, respectively. Figures~\ref{fig:kl_grid_sac} and~\ref{fig:kl_grid_madsac} show the corresponding inter-run policy-divergence curves.

\begin{figure}[p]
    \centering
    \includegraphics[
        width=\linewidth,
        height=0.802\textheight,
        keepaspectratio
    ]{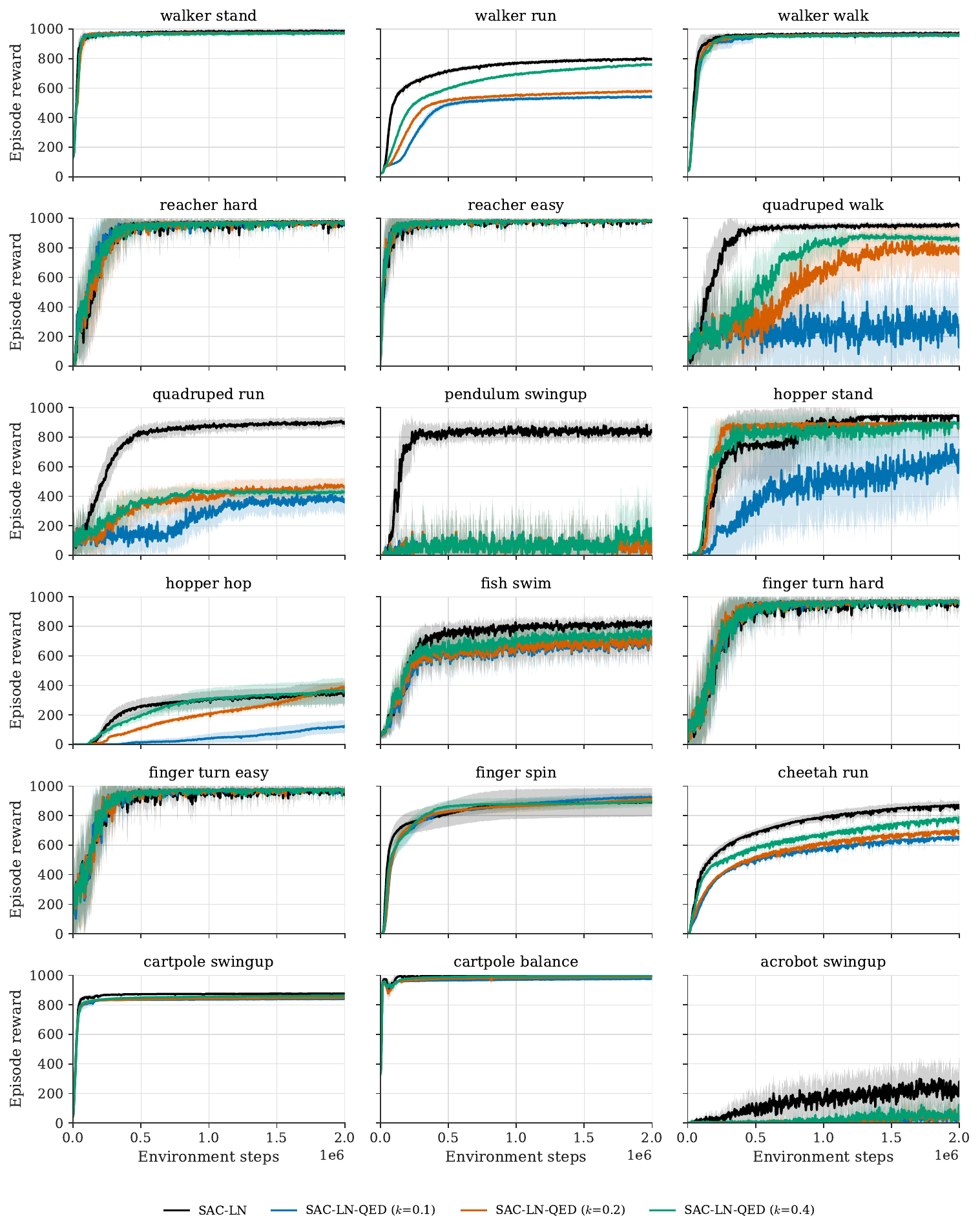}
    \caption{\textbf{Per-task learning curves for SAC-LN and SAC-QED on the 18-task \texttt{dm\_control} suite.}
    Episode reward over environment steps for the SAC-LN baseline and QED variants with $k\in\{0.1,0.2,0.4\}$, matching the SAC-LN conditions used in the aggregate results in Figure~\ref{fig:dmc_returns}.
    QED generally preserves strong performance on easier tasks while introducing a performance-consistency trade-off on harder locomotion tasks, especially when the behavioral-consistency constraint is strongest.}
    \label{fig:reward_grid_sac}
\end{figure}

\begin{figure}[p]
    \centering
    \includegraphics[
        width=\linewidth,
        height=0.82\textheight,
        keepaspectratio
    ]{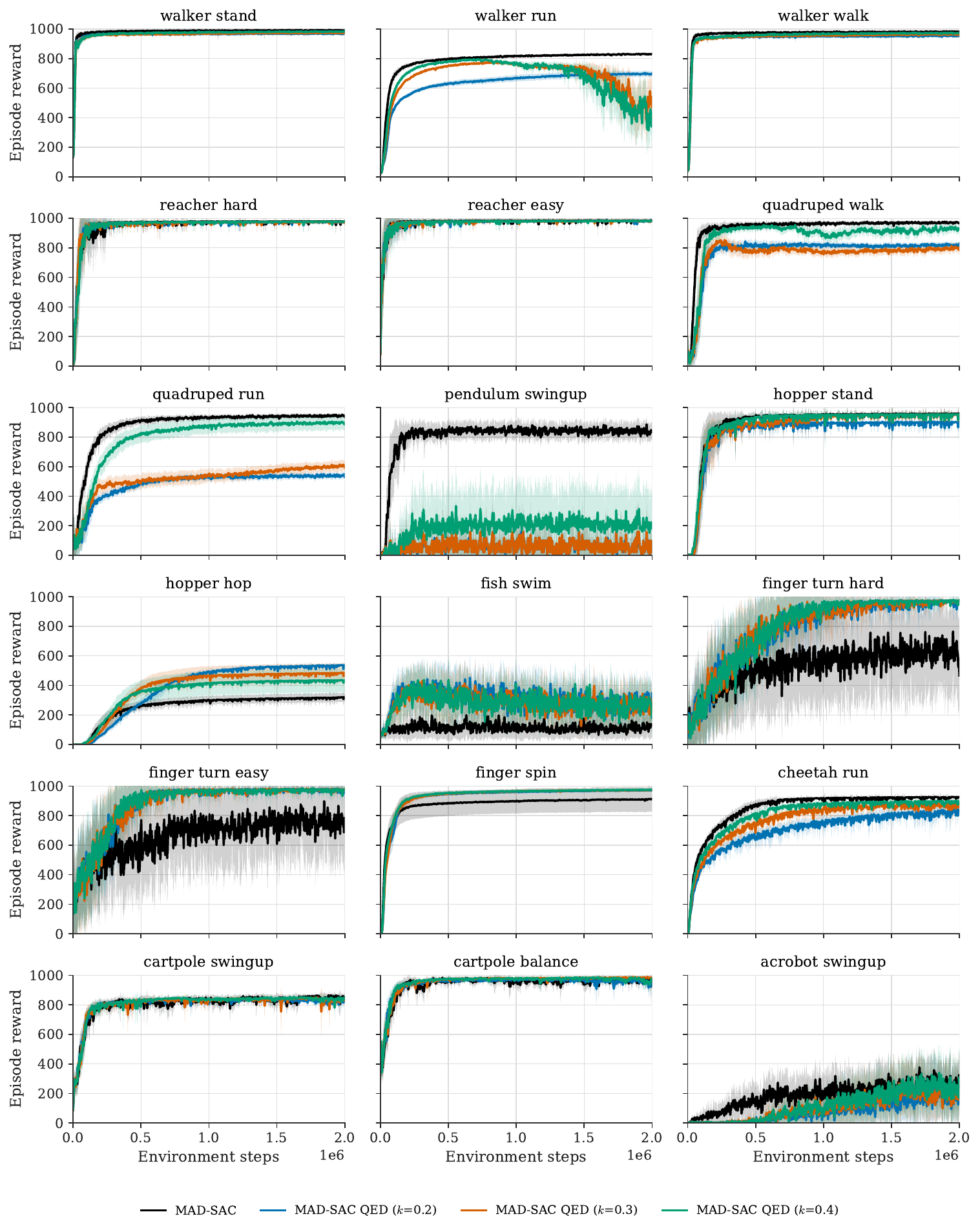}
    \caption{\textbf{Per-task learning curves for MAD-SAC and MAD-SAC-QED on the 18-task \texttt{dm\_control} suite.}
    Episode reward over environment steps for the MAD-SAC baseline and QED variants with $k\in\{0.2,0.3,0.4\}$, matching the MAD-SAC conditions used in the aggregate results in Figure~\ref{fig:dmc_returns}.
    Compared with SAC-LN, MAD-SAC is more robust to the additional entropy induced by QED, and QED often preserves or improves learning while reducing seed-level dispersion.}
    \label{fig:reward_grid_madsac}
\end{figure}

\begin{figure}[p]
    \centering
    \includegraphics[
        width=\linewidth,
        height=0.82\textheight,
        keepaspectratio
    ]{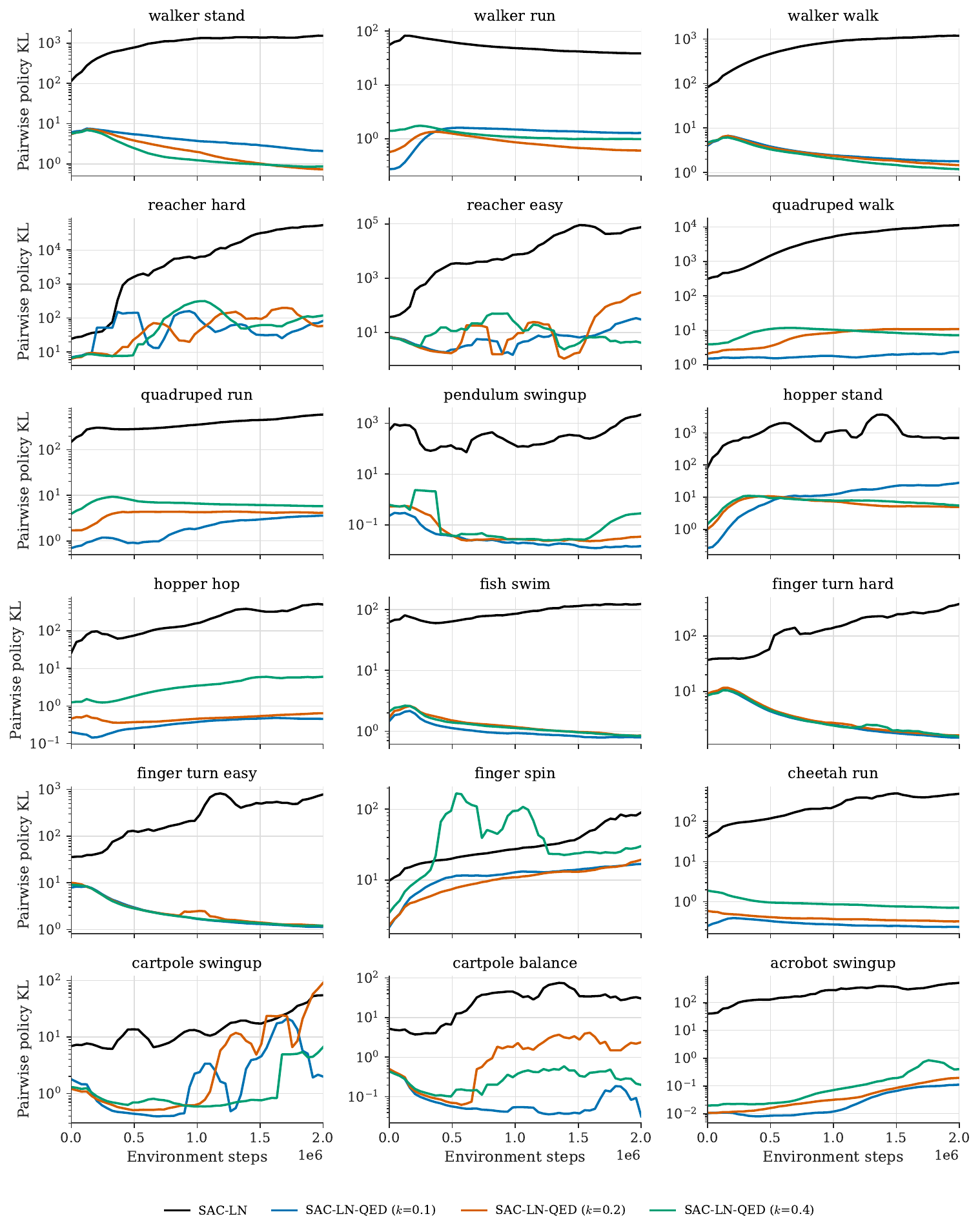}
    \caption{\textbf{Per-task inter-run policy divergence for SAC-LN and SAC-QED on the 18-task \texttt{dm\_control} suite.}
    Pairwise symmetric pre-tanh policy KL over environment steps for the SAC-LN baseline and QED variants with $k\in\{0.1,0.2,0.4\}$, matching the SAC-LN conditions used in the aggregate results in Figure~\ref{fig:dmc_returns}.
    Across most tasks, QED substantially lowers inter-run policy divergence relative to standard target-entropy tuning, with stronger behavioral-consistency constraints producing lower KL but sometimes larger performance costs.}
    \label{fig:kl_grid_sac}
\end{figure}

\begin{figure}[p]
    \centering
    \includegraphics[
        width=\linewidth,
        height=0.82\textheight,
        keepaspectratio
    ]{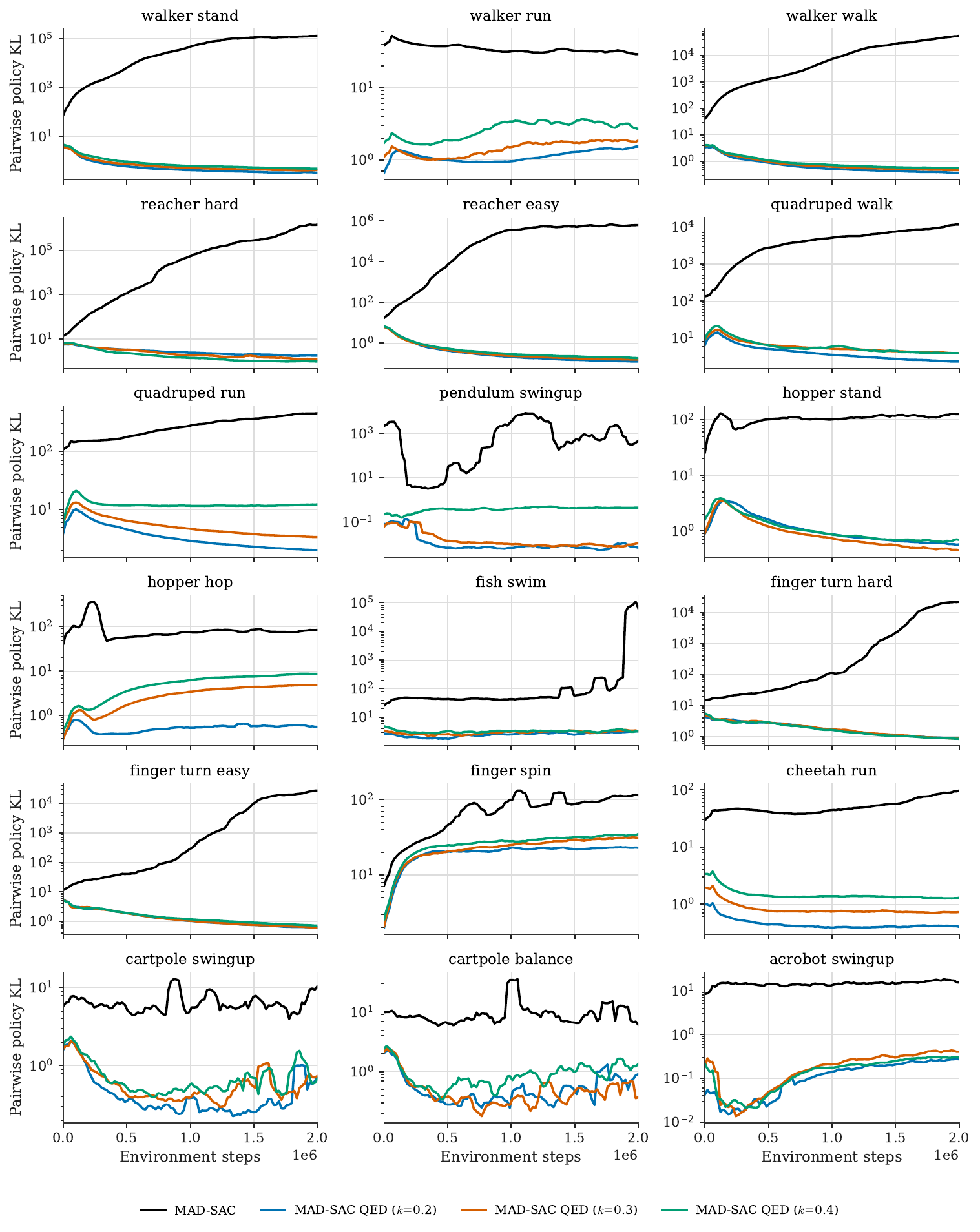}
    \caption{\textbf{Per-task inter-run policy divergence for MAD-SAC and MAD-SAC-QED on the 18-task \texttt{dm\_control} suite.}
    Pairwise symmetric pre-tanh policy KL over environment steps for the MAD-SAC baseline and QED variants with $k\in\{0.2,0.3,0.4\}$, matching the MAD-SAC conditions used in the aggregate results in Figure~\ref{fig:dmc_returns}.
    QED consistently suppresses the growth of cross-seed policy divergence, showing that the behavioral-consistency effect holds across tasks and is not driven only by the aggregate results.}
    \label{fig:kl_grid_madsac}
\end{figure}

\end{document}